\documentclass[twoside,11pt]{article}

\usepackage{jmlr2e}
\input{tcilatex}
\usepackage{amsmath}
\usepackage{amsfonts}
\usepackage{amssymb}
\usepackage{color}
\usepackage{algorithm}
\usepackage{algorithmic}
\usepackage{url}
\usepackage{enumerate}
\usepackage{booktabs}

\newcommand{\rb}{{\rangle}}
\newcommand{\lb}{{\langle}}

\newcommand{\R}{{\mathbb{R}}}

\newcommand{\X}{{\cal X}}

\newcommand{\beq}{\begin{equation}}
\newcommand{\eeq}{\end{equation}} 
\newcommand{\bea}{\begin{eqnarray}}
\newcommand{\eea}{\end{eqnarray}}
\input{tcilatex}

\let\sup\relax \DeclareMathOperator*\sup{\vphantom{p}sup}

\newcommand{\N}{{\mathbb N}}

\def\boldf#1{\hbox{\rlap{$#1$}\kern.4pt{$#1$}}}

\jmlrheading{v}{yyyy}{p-pp}{mm/yy}{mm/yy}{Andreas Maurer, Massimiliano Pontil and Bernardino Romera-Paredes}

\ShortHeadings{The Benefit of Multitask Representation Learning}{Maurer, Pontil and Romera-Paredes}
\firstpageno{1}

\begin{document}

\newtheorem{fact}[theorem]{Fact}

\title{The Benefit of Multitask Representation Learning}

\author{\name Andreas Maurer \email am@andreas-maurer.eu\\
\addr 
Adalbertstrasse 55, D-80799 M\"{u}nchen, Germany \\
\name Massimiliano Pontil \email massimiliano.pontil@iit.it\\
		\addr 
		Istituto Italiano di Tecnologia, 16163, Genoa, Italy \\ 
		Department of Computer Science, University College London, WC1E 6BT, UK \\
        \name Bernardino Romera-Paredes \email  bernard@robots.ox.ac.uk\\
        \addr Department of Engineering Science, University of Oxford, OX1 3PJ, UK\\
       }

\editor{Charles Sutton}

\maketitle

\begin{abstract}
We discuss a general method to learn data representations from multiple tasks. We provide a justification for this method in both settings of multitask learning and learning-to-learn. The method is illustrated in detail in the special case of linear feature learning. Conditions on the theoretical advantage offered by multitask representation learning over independent task learning are established. In particular, focusing on the important example of half-space learning, we derive the regime in which multitask representation learning is beneficial over independent task learning, as a function of the sample size, the number of tasks and the intrinsic data dimensionality. Other potential applications of our results include multitask feature learning in reproducing kernel Hilbert spaces and multilayer, deep networks.
\end{abstract}

\vspace{.2truecm}


\section{Introduction}
Multitask learning (MTL) can be characterized as the problem of learning multiple tasks {\em jointly}, as opposed to learning each task in isolation. This problem is becoming increasingly important due to its relevance in many applications, ranging from modelling users' preferences for products, to multiple object classification in computer vision, to patient healthcare data analysis in health informatics, to mention but a few.  Multitask learning algorithms which exploit structure and similarities across different learning problems have been studied by the machine learning community since the mid 90's, initially in connection to neural network models \citep[see][and reference therein]{Baxter 2000,Caruana 1998,Thrun 1998}. More recent approaches have been based on kernel methods \citep{Evgeniou 2005}, structured sparsity and convex optimization \citep{AEP}, among others.

Closely related to multitask learning but more challenging is the problem of {\em learning-to-learn} (LTL), namely learning to perform a new task by exploiting knowledge acquired when solving previous tasks.  Arguably, a solution to this problem would have major impact in Artificial Intelligence as we could build machines which learn from experience to perform new tasks, similar to what we observe in human behavior. 

An influential line of research on multitask and transfer learning is based on the idea that the tasks are related by means of a common low dimensional representation, which is learned jointly with the tasks' parameters. This approach was first advocated in \citep{Baxter 2000,Caruana 1998,Thrun 1998} and more recently reconsidered in \citep{AEP} from the perspective of convex optimization and sparsity regularization. Representation learning is also a key problem in AI, and in the past years there has been much renewed interest in learning nonlinear hierarchical representations from multiple tasks using multilayer, deep networks. Researchers have shown improved results in a number of empirical domains; the case of computer vision is perhaps most remarkable, \citep[see e.g.][and references therein]{Darrell}. This success has increased interest in multitask representation learning (MTRL) as it is a core component of deep networks. Still, the understanding of why this methodology works remains largely unexplored. 

In this paper we analyze a general method for MTRL and discuss its potential advantage in both the MTL setting, where the learned representation is applied to the same tasks used during training, and in the domain of LTL, where the representation is applied to new tasks. We derive upper bounds on the error of these methods and quantify their advantage over independent task learning. When the original data representation is high dimensional and the number of examples provided to solve a regression or classification problem is limited, any learning algorithm which does not use any sort of prior knowledge will perform poorly because there is not enough data to reliably estimate the model parameters. We make this statement precise by considering the example of half space learning.

\subsection{Previous Work}

Many papers have proposed multitask learning methods and studied their applications to specific problems \citep[see][and references therein]{Zhang 2005,AEP,Baxter 2000,Ben-David 2003,Caruana 1998,Cesa,KO13,MPR,PL14,Wilmer}. There is a vast literature on these subjects and the list of papers provided here is necessarily incomplete.

Despite the considerable success of multitask learning and in particular multitask representation learning there are only few theoretical investigations 
\citep{Zhang 2005,Baxter 2000,Ben-David 2003}. Other statistical learning bounds are restricted to linear multitask learning such as \citep{Cesa,Lounici 2011,Maurer 2005,Maurer 2005}. 

Learning-to-learn (also called inductive bias learning or transfer learning) has been proposed by \cite{Thrun 1998} and theoretically studied by \cite{Baxter 2000} where an error analysis is provided, showing that a common representation
which performs well on the training tasks will also generalize to new tasks obtained from the same ``environment''. More recent papers which present dimension independent bounds appear in \cite{Maurer 2005,Maurer 2006,MP13,MPR,PL14}. 

\subsection{Our Contributions}

There are two main contributions of this work. 
First we present bounds to both the MTL and LTL settings, which apply to a very general MTRL method. Our analysis goes well beyond linear representation learning considered in most previous works. It improves over the analysis by \cite{Baxter 2000} based on covering numbers. We use more recent  techniques of empirical
process theory to achieve bounds which are independent of the input dimension (hence also valid in reproducing kernel Hilbert spaces) and to avoid logarithmic factors. Furthermore our analysis can be made fully data dependent. When specialized to subspace learning (i.e. linear feature learning) we get best bounds valid for infinite dimensional input spaces. 

As the second main contribution of this paper, we explain the advantage of MTRL in terms of specificity of
feature maps and expose conditions when MTRL is beneficial or when it is not worth the effort. We further specialize our upper bounds to half-space learning (noiseless binary classification) and compare them to a general lower bound for learning isolated tasks. We observe that if the number of tasks grows then the performance of the method (both in the MTL and LTL setting) matches the performance of square norm regularization with best a priori known representation. This analysis highlights the advantage of multitask learning over learning the tasks independently. We also present numerical experiments for half-space learning, which indicate the good agreement between theory and experiments.

\subsection{Organization}

The paper is organized as follows. In Section \ref{sec:2}, we introduce the problem and present our main results.
In Section \ref{sec:SL}, we specialize these results to subspace learning and illustrate the role played by the data covariance matrices in our bounds. In Section \ref{subsec:HL} we further illustrate our results in the case of half-space learning, rigorously comparing our upper bounds to a general lower bound for orthogonal equivariant algorithms. In Section \ref{sec:proofs}, we present the proof of our main results, developing in particular uniform bounds on the estimation error. Finally, in Section \ref{sec:discussion} we summarize our findings and suggest directions for future research.  

\section{Multitask Representation Learning}
\label{sec:2}
The set of possible observations is denoted by $\mathcal{Z=}\left( \mathcal{X
},\mathbb{R}\right) $, where the members of $\mathcal{X}$ are interpreted as inputs and
the members of $\mathbb{R}$ are interpreted as outputs, or labels. A learning task is modelled by a
probability measure $\mu $ on $\mathcal{Z}$ where $\mu \left( x,y\right) $
is the probability to encounter the input-output pair $\left( x,y\right) \in 
\mathcal{Z}$ in the context of task $\mu $. We want to learn how to predict
outputs. If we predict $y$ while the true output is $y^{\prime }$, we suffer
a loss $\ell \left( y,y^{\prime }\right) $, where the loss function $\ell :\mathbb{R}\times 
\mathbb{R}\rightarrow \left[ 0,1\right] $ is assumed to be $1$-Lipschitz in the first
argument for every value of the second argument. Different Lipschitz
constants can be absorbed in the scaling of the predictors and different
ranges than $\left[ 0,1\right] $ can be handled by a simple scaling of our
results.

If $g$ is a real function defined on $\mathcal{X}$, then the values $g\left(
x\right) $ can be interpreted as predictors and the expectation $\mathbb{E}
_{\left( X,Y\right) \sim \mu }\left[ \ell \left( g\left( X\right) ,Y\right) 
\right] $ is the risk associated with hypothesis $g$ on the task $\mu $.

Multitask learning simultaneously considers many tasks $\mu _{1},\dots,\mu_{T}$ 
and hopes to exploit some suspected common property of these tasks.
For the purpose of this paper this property is the existence of a
representation or common feature-map, which simultaneously simplifies the
learning problem for most, or all of the tasks at hand. We consider
predictors $g$ which factorize 
\begin{equation*}
g=f\circ h,
\end{equation*}
where \textquotedblleft $\circ $\textquotedblright\ stands for functional
composition, that is, $(f\circ h)( x) =f\left( h\left( x\right) \right)$, for every $x \in \X$. The function $h:\mathcal{X\rightarrow \mathbb{R}}^{K}$ is called the
representation, or feature-map, and it is used across different tasks, while $
f$ is a function defined on $\mathcal{\mathbb{R}}^{K}$, a predictor
specialized to the task at hand. In the sequel $K$ will always be the
dimension of the representation space.

As usual in learning theory the functions $h:\mathcal{X}\rightarrow \mathbb{R}^{K}$ and $f:\mathbb{R}^K \rightarrow \mathbb{R}$ are chosen from respective hypothesis classes $\mathcal{H}$ and $\tciFourier $, which we refer to as the class of representations and the class of specialized predictors, respectively. These classes can be quite general, but we require that the functions in $\mathcal{\tciFourier }$ have Lipschitz constant at most $L$, for some positive real
number $L$.

The choice of representation and specialized predictors is based on the data
observed for all the tasks. This data takes the form of a multi-sample 
$\mathbf{\bar{Z}=}\left( \mathbf{Z}_{1},\dots,\mathbf{Z}_{T}\right) $, with 
$\mathbf{Z}_{t}=\left( Z_{t1},\dots,Z_{tn}\right) \sim \mu _{t}^{n}$. Here and
in the sequel an exponent on a measure indicates a product measure, so that 
$\mu _{t}^{n}$ is a measure on $\mathcal{Z}^{n}$ and $\mathbf{Z}_{t}$ is an
iid sample of $n$ random variables distributed as $\mu _{t}$. We also write 
$Z_{ti}=\left( X_{ti},Y_{ti}\right) $, $\mathbf{Z}_{t}=\left( \mathbf{X}_{t},
\mathbf{Y}_{t}\right) $ and $\mathbf{\bar{Z}=}\left( \mathbf{\bar{X},\bar{Y}}\right) $.

Multitask representation learning (MTRL) solves the optimization problem 
\begin{equation}
\min\left\{\frac{1}{nT}
\sum_{t=1}^{T}\sum_{i=1}^{n}\ell \left( f_{t}\left( h\left( X_{ti}\right)
\right) ,Y_{ti}\right): {h\in \mathcal{H},~(f_{1},\dots,f_{T})\in \tciFourier ^{T}}\right\} .  \label{the optimization problem}
\end{equation}
In this paper, we are not concerned with the algorithmics of this problem, but rather with the statistical properties of its solutions $\hat{h}
$ and $\hat{f}_{1},\dots,\hat{f}_{T}$. Note that these are functional random
variables in their dependence on $\mathbf{\bar{Z}}$.

We consider two possible applications of these solutions. One application, which we will refer to as multitask learning (MTL), retains both the representation $\hat{h}$ and the specializations $\hat{f}
_{1},\dots,\hat{f}_{T}$ to be applied to the tasks at hand. The other, perhaps
more important, application assumes that the tasks $\mu _{t}$ are related by
a probabilistic law, called an {\em environment}, and keeps only the
representation $\hat{h}$ to be used when specializing to new tasks obeying
the same law. In this way the parametrization of a learning algorithm is
learned, hence the name ``learning-to-learn" (LTL).
\bigskip

We will give general statistical guarantees in both cases. Our bounds
consist of three terms. The first term can be interpreted as the cost of
estimating the representation $h$ and decreases with the number $T$ of tasks
available for training. The second term corresponds to the cost of
estimating task-specific predictors and decreases with the number $n$ of
training examples available for each task. The last term contains the
confidence parameter and typically makes only a very small contribution.

It is not surprising that the complexity of the representation class 
$\mathcal{H}$ (first term in the bounds) plays a central role. We measure this complexity on the
observed input data $\mathbf{\bar{X}}\in \mathcal{X}^{Tn}$. Define a random
set $\mathcal{H}\left( \mathbf{\bar{X}}\right) \subseteq \mathbb{R}^{KTn}$ by
\begin{equation*}
\mathcal{H}\left( \mathbf{\bar{X}}\right) =\left\{ \left( h_{k}\left(
X_{ti}\right) \right) :h\in \mathcal{H}\right\} \text{.}
\end{equation*}
The complexity measure relevant to estimation of the representation is the
Gaussian average 
\begin{equation}
G\left( \mathcal{H}\left( \mathbf{\bar{X}}\right) \right) =\mathbb{E}\left[
\sup_{h\in \mathcal{H}}\sum_{kti}\gamma _{kti}h_{k}\left( X_{ti}\right)
|X_{ti}\right] ,  \label{Gaussian average of H(x)}
\end{equation}
where the $\gamma _{kti}$ are independent standard normal variables. The
Gaussian average is of order $\sqrt{nT}$ in $T$ and $n$ for many classes of interest. These include
kernel machines with Lipschitz kernels (e.g. Gaussian RBF) and arbitrarily
deep compositions thereof, see \cite{Maurer 2014} for a discussion. 
As we shall see, this increase of $O(\sqrt{nT})$ is compensated in our bounds and the cost of learning the representation vanishes in the multi-task limit $T\rightarrow \infty$.

The second term in the bounds is governed by the quantity
\begin{equation}
\sup_{h\in \mathcal{H}}\frac{1}{n\sqrt{T}}\left\Vert h\left( \mathbf{\bar{X}}
\right) \right\Vert =\frac{1}{\sqrt{n}}\sup_{h\in \mathcal{H}}\sqrt{\frac{1}{
nT}\sum_{kti}h_{k}\left( X_{ti}\right) ^{2}}
\label{specific representations}
\end{equation}
or an equivalent distribution-dependent expression. If the feature-maps in $
\mathcal{H}$ are very specific, in the sense that their components are
appreciably different from zero only for very special data, the quantity in 
\eqref{specific representations} can become much smaller than $1/\sqrt{n}$, a
phenomenon which can give a considerable competitive edge to MTRL, in particular if the per-task sample size $n$ is small. We will
demonstrate this in Section \ref{sec:SL}, where we apply Theorems \ref{thm:MTL} and \ref{thm:LTL} to
subspace-learning and show that the above quantity is related to the
operator norm of the data covariance.

\subsection{Bounding the Excess Task-averaged Risk (MTL)}

If we make no further assumptions on the generation of the task-measures $\mu
_{1},\dots,\mu _{T}$, a conceptually simple performance measure for a
representation $h$ and specialized predictors $f_{1},\dots,f_{T}$ is the
task-averaged risk 
\begin{equation*}
\mathcal{E}_{\mathrm{avg}}\left( h,f_{1},\dots,f_{T}\right) =\frac{1}{T}
\sum_{t=1}^{T}\mathbb{E}_{\left( X,Y\right) \sim \mu _{t}}\ell \left(
f_{t}\left( h\left( X\right) \right) ,Y\right) .
\end{equation*}
We want to compare this to the very best we can do using the classes 
$\mathcal{H}$ and $\tciFourier $, given complete knowledge of the
distributions $\mu _{1},\dots,\mu _{T}$. The minimal risk is clearly 
\begin{equation*}
\mathcal{E}_{\mathrm{avg}}^{\ast }=\min_{h\in \mathcal{H},\left(
f_{1},\dots,f_{T}\right) \in \mathbf{\tciFourier }^{T}}\mathcal{E}_{\mathrm{avg}}
\left( h,f_{1},\dots,f_{T}\right) \text{.}
\end{equation*}
It is a fundamental hope underlying our approach that the classes $\mathcal{H}$ 
and $\tciFourier $ are large enough for this quantity to be sufficiently
small for practical purposes. We use the words \textquotedblleft hope" and
\textquotedblleft belief" because an \textquotedblleft assumption" would
imply a statement to be used in analytical reasoning. Instead our approach
is agnostic, and our results are valid independent of the size of the
minimal risk above.

Our first result bounds the excess average risk, which measures the
difference between the task-averaged true risk of the solutions to (\ref{the
optimization problem}) and the theoretical optimum above. 

\begin{theorem}
Let $\mu _{1},\dots,\mu _{T}$, $\mathcal{H}$ and $\tciFourier $ be as above,
and assume $0\in \mathcal{H}$ and $f\left( 0\right) =0$ for all $f\in
\tciFourier $. Then for $\delta >0$ with probability at least $1-\delta $ in
the draw of $\mathbf{\bar{Z}\sim }\prod_{t=1}^{T}\mu _{t}^{n}$ we have that 
\begin{multline*}
\mathcal{E}_{\mathrm{avg}}(\hat{h},\hat{f}_{1},\dots,\hat{f}_{T})-\mathcal{E}_{\mathrm{avg}}^{\ast } \\
\leq \frac{c_{1}L~G\left( \mathcal{H}\left( \mathbf{\bar{X}}\right) \right) 
}{nT}+\frac{c_{2}Q\sup_{h\in \mathcal{H}}\left\Vert h\left( \mathbf{\bar{X}}
\right) \right\Vert }{n\sqrt{T}}+\sqrt{\frac{8\ln \left( 4/\delta \right) }{nT}},
\end{multline*}
where $c_{1}$ and $c_{2}$ are universal constants, $G(\mathcal{H}(\mathbf{\bar{X}}))$ 
is the Gaussian average in Equation 
\eqref{Gaussian average of
H(x)}, and $Q$ is the quantity 
\begin{equation}
Q\equiv Q({\cal F})\sup_{y\neq y^{\prime }\in \mathbb{R}^{Kn}}\frac{1}{\left\Vert y-y^{\prime
}\right\Vert }\mathbb{E}\sup_{f\in \tciFourier }\sum_{i=1}^{n}\gamma
_{i}\left( f\left( y_{i}\right) -f\left( y_{i}^{\prime }\right) \right) .
\label{eq:Q}
\end{equation}
\label{thm:MTL}
\end{theorem}

\noindent Remarks:

\begin{enumerate}
\item The assumptions $0\in \mathcal{H}$ and $f\left( 0\right) =0$ for all 
$f\in \tciFourier $ are made to give the result a simpler appearance. They
are not essential, as the reader can verify from the proof.

\item If $G\left( \mathcal{H}\left( \mathbf{\bar{x}}\right) \right) $ is of
order $\sqrt{nT}$ then the first term on the right hand side above is of order $1/\sqrt{Tn}$ 
and vanishes in the \textit{multi-task limit} $T\rightarrow \infty $
even for small values of $n$.

\item For reasonable classes $\tciFourier $ one can find a bound on $Q$,
which is independent of $n$, because the $\left\Vert y-y^{\prime}\right\Vert$ 
in the denominator balances the Gaussian average depending on the class $\tciFourier $. 

\item The quantity $\sup_{h}\left\Vert h\left( \mathbf{\bar{X}}\right)
\right\Vert $ is of order $\sqrt{nT}$ whenever $\mathcal{H}$ is uniformly
bounded, a crude bound being $\sqrt{nT}\sup_{h\in \mathcal{H}
}\max_{ti}\left\Vert h\left( x_{ti}\right) \right\Vert $. The second term is
thus typically of order $1/\sqrt{n}$. As explained in the discussion of Equation \eqref{specific representations} above it can be very small if the representation components in $\mathcal{H}$ are very data-specific.
\end{enumerate}

\subsection{Bounding the Excess Risk for Learning-to-learn (LTL)}

Now we consider the case where we only retain the representation $\hat{h}$
obtained from (\ref{the optimization problem}) and specialize it to future,
hitherto unknown tasks. This is of course only possible, if there is some
common law underlying the generation of tasks. Following \cite{Baxter 2000}
we suppose that the tasks originate in a common environment $\eta $, which
is by definition a probability measure on the set of probability measures on 
$\mathcal{Z}$. The draw of $\mu \sim \eta $ models the encounter of a
learning task $\mu $ in the environment $\eta $.

The environment $\eta $ induces a measure $\mu _{\eta }$ on $\mathcal{Z}$ by 
\begin{equation*}
\mu _{\eta }\left( A\right) =\mathbb{E}_{\mu \sim \eta }\left[ \mu \left(
A\right) \right]~\text{for}~A\subseteq \mathcal{Z}.
\end{equation*}%
This simple mixture plays an important role in the interpretation of our
results.

The measure $\eta$ also induces a measure $\rho_{\eta}$ on $\mathcal{Z}^{n}$ which corresponds to the draw of an $n$-sample from a random task in
the environment. To draw a sample $\mathbf{Z\in }\mathcal{Z}^{n}$ from $\rho
_{\eta }$ we first draw a task $\mu $ from $\eta $ and then generate the
sample $\mathbf{Z=}\left( Z_{1},\dots,Z_{T}\right) $ from $n$ independent
draws from $\mu $. Formally 
\begin{equation*}
\rho _{\eta }\left( A\right) =\mathbb{E}_{\mu \sim \eta }\left[ \mu
^{n}\left( A\right) \right] \text{ for }A\subseteq \mathcal{Z}^{n}\text{.}
\end{equation*}%
We assume that the tasks $\mu _{1},\dots ,\mu _{T}$ are drawn independently
from $\eta $ and, consequently, that the multisample $\mathbf{\bar{Z}}%
=\left( \mathbf{Z}_{1}\mathbf{,\dots,T}_{T}\right) $ is obtained in $T$
independent draws from $\rho _{\eta }$, that is, $\mathbf{\bar{Z}\sim }\rho
_{\eta }^{T}$.

The way we plan to use a representation $h\in \mathcal{H}$ on a new task $%
\mu \sim \eta $ is as follows: we draw a training sample $\mathbf{Z}=\left(
Z_{1},\dots,Z_{n}\right) $ from $\mu ^{n}$ and solve the optimization problem%
\begin{equation*}
\min_{f\in \tciFourier }\frac{1}{n}\sum_{i=1}^{n}\ell \left( f\left( h\left(
X_{i}\right) \right) ,Y_{i}\right) .
\end{equation*}%
Let $\hat{f}_{h,\mathbf{Z}}$ denote the minimizer and $m_{h,\mathbf{Z}}$ the
corresponding minimum. We will then use the hypothesis $a\left( h\right) _{%
\mathbf{Z}}=\hat{f}_{h,\mathbf{Z}}\circ h=\hat{f}_{h,\mathbf{Z}}(h(\cdot))$ for the new task. In this way any representation $%
h\in \mathcal{H}$ parametrizes a learning algorithm, which is a function $a(h):\mathcal{Z}^{n} \rightarrow\tciFourier \circ h$, defined, for every $\mathbf{Z\in }\mathcal{Z}^{n}$, as
\begin{equation*}
a\left( h\right) _{%
\mathbf{Z}}=\hat{f}_{h,\mathbf{Z}}\circ h.
\end{equation*}%
In this sense the problem of optimizing such a representation can properly
be called ``learning-to-learn". It can also be interpreted
as ``learning a hypothesis space" as in \citep{Baxter 2000}, namely
selecting a hypothesis space $\tciFourier \circ h$ from the collection of
hypothesis spaces $\left\{ \tciFourier \circ h:h\in \mathcal{H}\right\} $.

We can test the algorithm $a\left( h\right) $ on the environment $\eta $ in
the following way:

\begin{itemize}
\item we draw a task $\mu \sim \eta,$

\item we draw a sample $\mathbf{Z}\in \mathcal{Z}^{n}$ from $\mu ^{n},$

\item we run the algorithm to obtain $a\left( h\right) _{\mathbf{Z}}=\hat{f}%
_{h,\mathbf{Z}}\circ h$,

\item finally, we measure the loss of $a\left( h\right) _{\mathbf{Z}}$ on a random
data-point $Z=\left( X,Y\right) \sim \mu$.
\end{itemize}

To define the risk $\mathcal{E}_{\eta }\left( h\right) $ associated with the
algorithm $a\left( h\right) $ parametrized by $h$ we just replace all random
draws with corresponding expectations, so%
\begin{equation*}
\mathcal{E}_{\eta }\left( h\right) =\mathbb{E}_{\mu \sim \eta }\mathbb{E}_{%
\mathbf{Z\sim }\mu ^{n}}\mathbb{E}_{\left( X,Y\right) \sim \mu }\left[ \ell
\left( a\left( h\right) _{\mathbf{Z}}\left( X\right) ,Y\right) \right] .
\end{equation*}%
The best value for any representation $h$ in $a\left( h\right) ,$ given
complete knowledge of the environment, is then%
\begin{equation*}
\min_{h\in \mathcal{H}}\mathcal{E}_{\eta }\left( h\right) .
\end{equation*}%
But, given complete knowledge of the environment, this is still not the best
we can do using the classes $\tciFourier $ and $\mathcal{H}$, because for
given $\mu $ and $h$ we still use the expected performance $\mathbb{E}_{%
\mathbf{Z\sim }\mu ^{n}}\mathbb{E}_{Z\sim \mu }~\ell \left( a\left( h\right)
_{\mathbf{Z}}\left( X\right) ,Y\right) $ of the empirical risk minimization
algorithm $a\left( h\right) $, instead of using knowledge of $\mu $ to
replace it by $\min_{f\in \tciFourier }\mathbb{E}_{Z\sim \mu }\ell \left(
f\left( h\left( X\right) \right) ,Y\right) $. The very best we can do is
thus 
\begin{equation*}
\mathcal{E}_{\eta }^{\ast }=\min_{h\in \mathcal{H}}\mathbb{E}_{\mu \sim \eta
}\left[ \min_{f\in \tciFourier }\mathbb{E}_{Z\sim \mu }\ell \left( f\left(
h\left( X\right) \right) ,Y\right) \right] .
\end{equation*}%
The excess risk associated with any representation $h$ is thus 
\begin{equation*}
\mathcal{E}_{\eta }\left( h\right) -\mathcal{E}_{\eta }^{\ast }.
\end{equation*}%
We give the following bound for the excess risk associated with the
representation $\hat{h}$ found as solution to the optimization problem (\ref%
{the optimization problem}).

\begin{theorem}
Let $\eta $ be an environment on $\mathcal{Z}$ and $\mathcal{H}$ and $%
\tciFourier $ as above. Then: (i) with probability at least $1-\delta $ in the
draw of $\mathbf{\bar{Z}}\sim \rho _{\eta }^{T}$ 
\vspace{.2truecm}
$$
\mathcal{E}_{\eta }( \hat{h}) -\mathcal{E}_{\eta }^{\ast }\leq \frac{\sqrt{2\pi }L~G\left( \mathcal{H}\left( \mathbf{\bar{X}}\right)
\right) }{T\sqrt{n}}+\sqrt{2\pi }Q^{\prime }\sup_{h\in \mathcal{H}}\sqrt{%
\frac{\mathbb{E}_{\left( X,Y\right) \sim \mu _{\eta }}\left[ \left\Vert
h\left( X\right) \right\Vert ^{2}\right] }{n}}+\sqrt{\frac{8\ln \left(
4/\delta \right) }{T}},
$$
\vspace{.2truecm}
and (ii) with the same probability
\vspace{.2truecm}
$$
\mathcal{E}_{\eta }( \hat{h}) -\mathcal{E}_{\eta }^{\ast } \leq \frac{\sqrt{2\pi }L~G\left( \mathcal{H}\left( \mathbf{\bar{X}}\right)
\right) }{T\sqrt{n}}+\frac{\sqrt{2\pi }Q^{\prime }\left( 1/T\right)
\sum_{t}\sup_{h\in \mathcal{H}}\left\Vert h\left( \mathbf{X}_{t}\right)
\right\Vert }{n}+5\sqrt{\frac{\ln \left( 8/\delta \right) }{T}},
$$
\vspace{.2truecm}
where $\hat{h}$ is solution to the problem (\ref{the optimization problem}), 
$G\left( \mathcal{H}\left( \mathbf{\bar{X}}\right) \right) $ is the Gaussian
average introduced in (\ref{Gaussian average of H(x)}), and $Q^{\prime }$ is
the quantity%
\begin{equation}
Q^{\prime }\equiv Q^{\prime }({\cal F}) = \sup_{y\in \mathbb{R}^{Kn}\backslash \left\{ 0\right\} }\frac{1}{%
\left\Vert y\right\Vert }\mathbb{E}\sup_{f\in \tciFourier
}\sum_{i=1}^{n}\gamma _{i}f\left( y_{i}\right) .
\label{eq:Qp}
\end{equation}
\label{thm:LTL}
\end{theorem}

We make some remarks and comparison to the previous result.
\begin{enumerate}
\item The constants are now explicit and small. For Theorem \ref{thm:MTL}, 
uniform estimation had to be controlled simultaneously in $\mathcal{H}$ and $%
\tciFourier $, while for LTL the problem can be more easily
decoupled.

\item The first term is equivalent to the first term in Theorem \ref{thm:MTL}
except for $\sqrt{n}$ replacing $n$ in the denominator. It is therefore
typically of order $1/\sqrt{T}$ instead of $1/\sqrt{nT}$. The different
order is due to the estimation of a hitherto unknown task, for which the
sample sizes are irrelevant. To understand this point assume the $\eta $ has
the property that every $\mu \sim \eta $ is deterministic, that is supported
on a single point $z_{\mu }\in \mathcal{Z}$. Then clearly the sample size $n$
is irrelevant, and the problem becomes equivalent to learning a single task
with a sample of size $T$.

\item The quantity $Q^{\prime }$ is very much like the quantity $Q$ in Equation \eqref{eq:Q}, and it is uniformly bounded in $n$ for the classes we consider. For linear
classes $Q=Q^{\prime }$.

\item The bound in part (i) is not fully data-dependent, but more convenient
for our applications below. The quantity 
\begin{equation*}
\sup_{h\in \mathcal{H}}\sqrt{\mathbb{E}_{\left( X,Y\right) \sim \mu _{\eta
}}\left\Vert h\left( X\right) \right\Vert ^{2}}=\sup_{h\in \mathcal{H}}\sqrt{%
\sum_{k}\mathbb{E}_{\left( X,Y\right) \sim \mu _{\eta }}\left[ h_{k}\left(
X\right) ^{2}\right] }
\end{equation*}%
plays a similar role to (\ref{specific representations}), which is its
empirical counterpart. Again, if the features $h_{k}$ are very specific, as the
dictionary atoms of the next section or the atoms in a radial basis function
network, then the above quantity can become very small.
\end{enumerate}

\subsection{Comparison to Previous Bounds}

The first and most important theoretical study of MTL and LTL was carried
out by \cite{Baxter 2000}, where sample complexity bounds are given for both settings. Instead of a feature map a hypothesis space is selected from a
class of hypothesis spaces. Clearly every feature map with values in $\mathbb{R}^{K}$ 
defines a hypothesis space while the reverse is not true in general,
so Baxter's setting is certainly more general than ours. On the other hand
the practical applications discussed in \citep{Baxter 2000} can be cast in the
language of feature learning.

To prove his sample complexity bounds Baxter uses covering numbers. This
classical method requires to cover a (meta-)hypothesis space (or its
evaluation on a sample) with a set of balls in an appropriately chosen
metric. The uniform bound is then obtained as a union bound over the cover
and bounds valid on the individual balls. The latter bounds follow from
Lipschitz properties $L$ of the loss function relative to the chosen metric.
For a bound of order $\epsilon $ the radius of the balls has to be of order 
$\epsilon /L$. This leads to covering numbers of order $\epsilon ^{-d}$,
where $d$ is some exponent (see the last inequalities in the proof of
 in \citep{Baxter 2000}, and has the consequence that the dominant term
in the bound has an additional factor of $\ln \left( 1/\epsilon \right) $.
This is manifest in Theorem 8, Theorem 12 and Corollary 13 in \citep{Baxter 2000} 
and constitutes an essential weakness of the method of covering numbers. For
bounds on the excess risk it implies that the orders of $\sqrt{1/T}$ and $
\sqrt{1/n}$ obtained from Rademacher or Gaussian complexities have to be
replaced by $\sqrt{\ln \left( T\right) /T}$ and $\sqrt{\ln \left( n\right) /n
}$.

Rademacher and Gaussian complexities make it easy to handle infinite
dimensional input spaces (see our Theorems \ref{Theorem MTL for subspace} and \ref{Theorem LTL for subspace} below). 
They also lead to data dependent bounds, which allows us to explain the
benefits of multi-task learning in terms of the spectrum of the data
covariance operator and the effective input dimension. Bounding Gaussian
complexities for linear classes is comparatively simple, see the proof of
our Lemma \ref{lemma:3}. There is a wealth of recent literature on the Rademacher
complexity of matrices with spectral regularizers \citep[see e.g.][and references therein]{KakadeEtAl 2012,MP13}, while it is unclear to us
how Baxter's method could be applied if the feature map is constrained by a
bound on, say, the trace norm of the associated matrix. In the case of LTL, our approach also leads to explicit and small constant
factors.

On the other hand it must be admitted, that it is relatively easy to obtain
bounds (also provided by Baxter) of order $\ln \left( n\right) /n$ or $\ln
\left( T\right) /T$ with covering numbers in the realizable case. Such
bounds would be more difficult to obtain with our techniques.

The work of \cite{Zhang 2005} proposes the use of MTL as a method of
semi-supervised learning through the creation of artificial tasks from
unlabelled data, for example predicting concealed components of vectors.
They analyze a specific algorithm where the class of feature maps can be
seen as a linear mixture of a fixed feature map with subspace projections as
discussed in our paper. The bounds given apply to the task-averaged risk and
not to LTL. The analysis is based on Rademacher averages and
is independent of the input dimension. The bound itself is expressed as an
entropy integral as given by Dudley \citep[see e.g.][]{Van Der Vaart and Wellner} but it is not very explicit. In
particular the role of the spectrum of the data covariance is not apparent.

\section{Multi-task Subspace Learning}

\label{sec:SL}

We illustrate the general results of the previous section with an important
special case. We assume that the input space $\mathcal{X}$ is a bounded subset of a Hilbert space $H$, which could for example be a reproducing kernel Hilbert space. We denote by $\langle \cdot,\cdot\rangle$ the inner product in $H$ and by $\|\cdot\|$ the induced norm. We hope that sufficiently good results can be
obtained by predictors of the form $g$, where $g:H\rightarrow \mathbb{R}$ is
linear with bounded norm. We also suspect that only few linear features in $H
$ suffice for most tasks, so that the vectors defining the hypotheses $g$
can all be chosen from one and the same, albeit unknown, $K$-dimensional
subspace $M$ of $H$.

Consequently we will factorize predictors as $f\circ h$, where $h$ is a
partial isometry $h:H\rightarrow \mathbb{R}^{K}$ and $f$ is a linear
functional on $\mathbb{R}^{K}$ chosen from some ball of bounded radius. Specifically, we introduce the classes 
\begin{eqnarray*}
\mathcal{H} &=&\left\{ H \ni x \mapsto \left( \left\langle
d_{1},x\right\rangle ,\dots,\left\langle d_{K},x\right\rangle \right) \in 
\mathbb{R}^{K}~:~D=\left( d_{1},\dots,d_{K}\right) \in H^{K}\text{ orthonormal}%
\right\} \\
\tciFourier &=&\left\{ \mathbb{R}^{K} \ni y \mapsto \sum_{k}w_{k}y_{k}\in 
\mathbb{R}~:~\sum_{k}w_{k}^{2}\leq B^{2}\right\} .
\end{eqnarray*}%
The $D$'s appearing in the definition of $\mathcal{H}$ are also called
dictionaries and the individual $d_{k}$ are called atoms \citep[see][]{MPR}.

It does no harm to our analysis if we immediately generalize the class $%
\mathcal{H}$ so as to include certain two-layer neural networks by allowing
a nonlinear activation function $\phi $ with Lipschitz constant $L_{\phi }$
and satisfying $\phi \left( 0\right) =0$, to be applied with each atom. We
can also drop the condition of orthonormality and allow the atoms to trade
some of their norms when needed. The enlarged class of representations is 
\begin{equation*}
\mathcal{H}=\left\{ x\in H\mapsto \left( \phi \left( \left\langle
d_{1},x\right\rangle \right) ,\dots,\phi \left( \left\langle
d_{K},x\right\rangle \right) \right) \in \mathbb{R}^{K}:d_{1},\dots,d_{K}\in
H,~\sum_{k}\left\Vert d_{k}\right\Vert ^{2}\leq K\right\} .
\end{equation*}%
The results can then be re-specialized to subspace learning by setting $\phi 
$ to the identity and $L_{\phi }$ to one.

When applied to subspace learning, our bounds are expressed in terms of
covariances. If $\nu $ is a probability measure on $H$ the corresponding
covariance operator $C_{\nu }$ is defined by%
\begin{equation*}
\left\langle C_{\nu }v,w\right\rangle =\mathbb{E}_{X\sim \nu }\left\langle
v,X\right\rangle \left\langle X,w\right\rangle \text{ for }v,w\in H\text{.}
\end{equation*}%
For an environment $\eta $ we denote the covariance operator corresponding
to the data-marginal of the mixture measure $\mu _{\eta }$ simply by $C$.

If $\mathbf{x}=\left( x_{1},\dots,x_{m}\right) \in H^{m}$ we define the
empirical covariance operator $\hat{C}\left( \mathbf{x}\right) $ by%
\begin{equation*}
\left\langle \hat{C}\left( \mathbf{x}\right) v,w\right\rangle =\frac{1}{m}%
\sum_{i}\left\langle v,x_{i}\right\rangle \left\langle x_{i},w\right\rangle 
\text{ for }v,w\in H\text{,}
\end{equation*}%
in particular 
\begin{equation*}
\left\langle \hat{C}\left( \mathbf{\bar{X}}\right) v,w\right\rangle =\frac{1%
}{nT}\sum_{ti}\left\langle v,X_{ti}\right\rangle \left\langle
X_{ti},w\right\rangle \text{.}
\end{equation*}

The following lemma establishes the necessary ingredients for the
application of Theorems \ref{thm:MTL} and \ref{thm:LTL} to the case of
subspace learning. Recall that if $A$ is a selfadjoint positive linear operator on ${H}$, we denote by $\|A\|_\infty$ and $\|A\|_1$ its spectral and trace norms, respectively. They are defined as $\|A\|_\infty = \sup_{\|z\|\leq 1} \|Az\|$ and $\|A\|_1 = \sum_{i\in \N} \langle e_i, A e_i\rangle$, where $\{e_i\}_{i \in \N}$ is an orthonormal basis in $H$. Recall also the definition of $Q\left( \tciFourier \right)$ and $Q'\left( \tciFourier \right)$ given in Equations \eqref{eq:Q} and \eqref{eq:Qp}, respectively.

\begin{lemma}
\label{lemma:3}
Let $\mathbf{\bar{x}}=\left( x_{ti}\right) $ be a $T\times n$ matrix with
values in a Hilbert space and let $\phi $, $\mathcal{H}$ and $\tciFourier $
be defined as above. Then 

(i) $G\left( \mathcal{H}\left( \mathbf{\bar{x}}\right) \right) \leq L_{\phi
}K\sqrt{nT\left\Vert \hat{C}\left( \mathbf{\bar{x}}\right) \right\Vert _{1}}.$

(ii) For every $h\in \mathcal{H}$,~$\left\Vert h\left( \mathbf{\bar{x}}\right)
\right\Vert \leq L_{\phi }\sqrt{KnT\left\Vert \hat{C}\left( \mathbf{\bar{x}}
\right) \right\Vert _{\infty }}.$

(iii) For an environment $\eta $ and every $h\in \mathcal{H}$
\begin{equation*}
\mathbb{E}_{\left( X,Y\right) \sim \mu _{\eta }}\left\Vert h\left( X\right)
\right\Vert ^{2}\leq L_{\phi }^{2}K\left\Vert C\right\Vert _{\infty }.
\end{equation*}

(iv) $L\left( \tciFourier \right) \leq B$.

(v) $Q\left( \tciFourier \right) \leq B$ and $Q^{\prime }\left( \tciFourier
\right) \leq B$.
\end{lemma}
\begin{proof}
\emph{(i)} Using the contraction lemma, Corollary \ref{Contraction Lemma}, in the first \ inequality and Cauchy-Schwarz and
Jensen's inequality in the second we get
\begin{eqnarray*}
G\left( \mathcal{H}\left( \mathbf{\bar{x}}\right) \right)  &\leq &L_{\phi }
\mathbb{E}\sup_{d\in \mathcal{H}}\sum_{kti}\gamma _{kti}\left\langle
d_{k},x_{ti}\right\rangle  \\
&=&L_{\phi }\mathbb{E}\sup_{d\in \mathcal{H}}\sum_{k}\left\langle
d_{k},\sum_{ti}\gamma _{kti}x_{ti}\right\rangle  \\
&\leq &L_{\phi }\sqrt{K}\left( \sum_{k}\mathbb{E}\left\Vert \sum_{ti}\gamma
_{ti}x_{ti}\right\Vert ^{2}\right) ^{1/2} \\
&\leq &L_{\phi }K\left( \sum_{ti}\left\Vert x_{ti}\right\Vert ^{2}\right)
^{1/2}=L_{\phi }K\sqrt{nT\left\Vert \hat{C}\left( \mathbf{x}\right)
\right\Vert _{1}}.
\end{eqnarray*}

\noindent \emph{(ii)} For any $D\in \mathcal{H}$
\begin{eqnarray*}
\sum_{kti}\phi \left( \left\langle d_{k},x_{ti}\right\rangle \right) ^{2}
&\leq &L_{\phi }^{2}\sum_{kti}\left\langle d_{k},x_{ti}\right\rangle ^{2} \\
&=&L_{\phi }^{2}\sum_{k}\left\Vert d_{k}\right\Vert
^{2}\sum_{ti}\left\langle \frac{d_{k}}{\left\Vert d_{k}\right\Vert }
,x_{ti}\right\rangle ^{2} \\
&\leq &L_{\phi }^{2}K\sup_{v:\left\Vert v\right\Vert \leq
1}\sum_{ti}\left\langle v,x_{ti}\right\rangle ^{2} \\
&=&L_{\phi }^{2}KnT\left\Vert \hat{C}\left( \mathbf{\bar{x}}\right)
\right\Vert _{\infty },
\end{eqnarray*}
where we used $\phi \left( 0\right) =0$ in the first step.

\noindent {\em (iii)} Similarly, we have that
\begin{eqnarray*}
\mathbb{E}_{\left( X,Y\right) \sim \mu _{\eta }}\sum_{k}\phi \left(
\left\langle d_{k},X\right\rangle \right) ^{2} &\leq &L_{\phi
}^{2}\sum_{k}\left\Vert d_{k}\right\Vert ^{2}\mathbb{E}_{\left( X,Y\right)
\sim \mu _{\eta }}\left\langle \frac{d_{k}}{\left\Vert d_{k}\right\Vert }
,X\right\rangle ^{2} \\
&\leq &L_{\phi }^{2}K\sup_{\left\Vert v\right\Vert \leq 1}\mathbb{E}_{\left(
X,Y\right) \sim \mu _{\eta }}\left\langle v,X\right\rangle ^{2} \\
&=&L_{\phi }^{2}K\left\Vert C\right\Vert _{\infty }.
\end{eqnarray*}

\noindent \emph{(iv)} Let $y,y^{\prime }\in \mathbb{R}^{K}$. Then
\begin{equation*}
\sup_{w\in \tciFourier }\left\{
\sum_{k}w_{k}y_{k}-\sum_{k}w_{k}y_{k}^{\prime }\right\} \leq \left(
\sum_{k}w_{k}^{2}\right) ^{1/2}\left\Vert y-y^{\prime }\right\Vert \leq
B\left\Vert y-y^{\prime }\right\Vert ,
\end{equation*}
so $L\leq B$.

\noindent \emph{(v)} Similarly, we have that
\begin{eqnarray*}
\mathbb{E}\sup_{w\in \mathcal{\tciFourier }}\sum_{i}\gamma _{i}\left(
\sum_{k}w_{k}y_{ki}-\sum_{k}w_{k}y_{ki}\right) &=&\mathbb{E}\sup_{w\in \mathcal{\tciFourier }}\sum_{k}w_{k}\sum_{i}\gamma
_{i}\left( y_{ki}-y_{ki}\right) \\
&\leq &\sup_{w\in \mathcal{\tciFourier }}\sqrt{ \sum_{k}w_{k}^{2} \sum_{k}\mathbb{E}\left( \sum_{i}\gamma _{i}\left(
y_{ki}-y_{ki}\right) \right) ^{2}} \\
&\leq &B\sqrt{ \sum_{ki}\left( y_{ki}-y_{ki}\right) ^{2}}=B\left\Vert y-y^{\prime }\right\Vert ,
\end{eqnarray*}%
so $Q\leq B$. The same proof works for $Q^{\prime }$.
\end{proof}

Substitution in Theorem \ref{thm:MTL} immediately gives
\begin{theorem}[subspace MTL]
\label{Theorem MTL for subspace}With probability at least $1-\delta $ in 
$\mathbf{\bar{X}}$ the excess risk is bounded by 
\begin{equation}
\mathcal{E}_{\mathrm{avg}}(\hat{h},\hat{f}_{1},\dots,\hat{f}_{T})-\mathcal{E}_{
\mathrm{avg}}^{\ast }\leq c_{1}L_{\phi }BK\sqrt{\frac{\left\Vert \hat{C}
\left( \mathbf{\bar{X}}\right) \right\Vert _{1}}{nT}}+c_{2}L_{\phi }B\sqrt{
\frac{K\left\Vert \hat{C}\left( \mathbf{\bar{X}}\right) \right\Vert _{\infty
}}{n}}+\sqrt{\frac{8\ln \left(2/\delta \right) }{nT}}.  \label{New MTL subspace bound}
\end{equation}
\end{theorem}

We remark that in the linear case the best competing bound for MTL, obtained by 
\cite{MP13} from noncommutative Bernstein inequalities, is 
\begin{equation}
2B\sqrt{\frac{K\left\Vert \hat{C}\left( \mathbf{\bar{X}}\right) \right\Vert
_{1}\ln \left( Tn\right) }{nT}}+B\sqrt{\frac{8K\left\Vert \hat{C}\left( 
\mathbf{\bar{X}}\right) \right\Vert _{\infty }}{n}}+\sqrt{\frac{8\ln
\left(2/\delta \right) }{nT}}.  \label{Tracenorm bound}
\end{equation}%
If we disregard the constants this is worse than the bound (\ref{New MTL
subspace bound}) whenever $K<\ln \left( Tn\right) $. Its approach to the
multitask limit is slower ($\sqrt{\ln \left( T\right) /T}$ as opposed to $\sqrt{1/T}$), but of course it has the advantage of smaller constants. The
methods used to obtain (\ref{Tracenorm bound}), however, break down for
nonlinear dictionaries.

For the LTL setting, we use the distribution dependent bound,
Theorem \ref{thm:LTL} (i), and obtain

\begin{theorem}[subspace LTL]
\label{Theorem LTL for subspace}With probability at least $1-\delta $ in $
\mathbf{\bar{X}}$, the excess risk is bounded by 
\begin{equation*}
\mathcal{E}_{\eta } (\hat{h}) -\mathcal{E}_{\eta }^{\ast }\leq 
\sqrt{2\pi }L_{\phi }B\left( \frac{K\sqrt{\left\Vert \hat{C}\left( \mathbf{
\bar{X}}\right) \right\Vert _{1}}}{\sqrt{T}}+\sqrt{\frac{K\left\Vert
C\right\Vert _{\infty }}{n}}\right) +\sqrt{\frac{8\ln \left(4/\delta \right)}{T}}.
\end{equation*}
\end{theorem}

The two most important common features of Theorems \ref{Theorem MTL for subspace} and \ref{Theorem LTL for subspace} are the decay to zero
of the first term, as $T\rightarrow \infty $, and the occurrence of the
operator norm of the empirical or true covariances in the second term. The first
implies that for very large numbers of tasks the bounds are dominated by the
second term.

To understand the second term we must first realize that the ratio of trace
and operator norms of the true covariances can be interpreted as an
effective dimension of the distribution. This is easily seen if the mixture
of task-marginals is concentrated and uniform on a $d$-dimensional
unit-sphere. In this case $\left\Vert C\right\Vert _{1}=1$ and by isotropy
all eigenvalues are equal, so $\left\Vert C\right\Vert _{\infty }=1/d$,
whence $\left\Vert C\right\Vert _{1}/\left\Vert C\right\Vert _{\infty }=d$.
In such a case the second term in Theorem \ref{Theorem LTL for subspace}
above becomes 
\begin{equation}
B\sqrt{\frac{K}{dn}}.  \label{second term subspace}
\end{equation}
The appropriate standard bound for learning the tasks independently would be 
$B\sqrt{1/n}$ \citep[see][]{Bartlett 2002}. The ratio $\sqrt{K/d}$ of the two
bounds in the multitask limit is the quotient of utilized information (the
dimension of the representation space) to available information (the
dimension of the data). This highlights the potential advantages of
MTRL: if the data is already low-dimensional in the order of $
K$ then multi-task learning isn't worth the extra computational labour. If
the data is high dimensional however, then multi-task learning may be
superior.

The expression (\ref{second term subspace}) above might suggest that there
really is a benefit of high dimensions for learning-to-learn. This is of
course not the case, because the regularizer $B$ has to be chosen large, in
fact proportional to $\sqrt{d}$ to allow a small empirical error. The
correct interpretation of (\ref{second term subspace}) is that the burden of
high dimensions vanishes in the limit $T\rightarrow \infty $. In the next
section we will explain this point in more detail.

\subsection{Learning to Learn Half-spaces}
\label{subsec:HL}
In this section, we illustrate the benefit of MTRL over independent task
learning (ITL) in the case of noiseless linear binary classification (or
half-space learning). We compare our upper bounds for LTL to a general
lower bound on the performance of ITL algorithms and quantify the
parameter regimes where LTL is superior to ITL.

We assume that all the input marginals are given by the uniform distribution 
$\sigma $ on the unit sphere $\mathcal{S}_{d}$ in $\mathbb{R}^{d}$, and the
objective is for each task $\mu $ to classify membership in the half-space $
\left\{ x:\left\langle x,u_{\mu }\right\rangle >0\right\} $ defined by a
task-specific (unknown) unit vector $u_{\mu }$. In the given environment all
the vectors $u_{\mu }$ are assumed to lie in some (unknown) $K$-dimensional
subspace $M$ of $\mathbb{R}^{d}$. We are interested in the regime that
\begin{equation*}
K\ll n\ll d
\end{equation*}
and $T$ grows. This is the safe regime in which our upper bounds for MTL or LTL (cf. Theorems 
\ref{Theorem MTL for subspace} and \ref{Theorem LTL for subspace}) are smaller than a uniform lower bound for independent 
task learning, which we discuss below. We need $n \ll d$ for the lower bound to be large and $K \ll n$ for the middle term in our upper 
bounds to be small. If $T$ is large enough, the second term in our upper bounds dominates the first (task dependent) term. A safe choice is $T \gg K^2 d$, see Equation \eqref{upper bound half-space learning} below.

The $0$-$1$-loss is unsuited for our bounds because it is not Lipschitz. Instead we
will use the truncated hinge loss with unit margin given by $\ell \left(
y^{\prime },y\right) =\xi \left( y^{\prime }y\right) $, where $\xi $ is the
real function 
\begin{equation*}
\xi \left( t\right) =\left\{ 
\begin{array}{lll}
1 & \text{if} & t\leq 0, \\ 
1-t & \text{if} & 0<t\leq 1, \\ 
0 & \text{if} & 1<t.
\end{array}
\right. 
\end{equation*}
This loss is an upper bound of the $0$-$1$-loss, so upper bounds for this
loss function are also upper bounds for the classification error. 

Let $\mathcal{H}$ and $\tciFourier $ be as given at the beginning of Section 
\ref{sec:SL} in its linear variant, where $\mathcal{H}$ is defined by
orthonormal dictionaries without activation functions. Thus, $\mathcal{H}$ can 
be viewed as the set of partial isometries $D:H\rightarrow \mathbb{R}^{K}$.

Recall the definition of the minimal risk for LTL
\begin{eqnarray*}
\mathcal{E}_{\eta }^{\ast } &=&\min_{h\in \mathcal{H}}\mathbb{E}_{\mu \sim
\eta }\left[ \min_{f\in \tciFourier }\mathbb{E}_{Z\sim \mu }\ell \Big(
f\big( h\left( X\right) \big) ,Y\Big) \right]  \\
&=&\min_{D\in \mathcal{H}}\mathbb{E}_{\mu \sim \eta }\left[ \min_{\left\Vert
w\right\Vert \leq B}\mathbb{E}_{Z\sim \mu }\xi \big( \left\langle
w,DX\right\rangle \text{sgn}\left( \left\langle
u_{\mu }, X \right\rangle
\right) \big) \right].
\end{eqnarray*}
Let $D_{M}$ be the partial isometry mapping $M$ onto $\mathbb{R}^{K}$. 
Then $D_{M}\in \mathcal{H}$ and for every unit vector $u\in H$ we
have $D_{M}\left( Bu\right) \in \tciFourier $. Thus
\begin{eqnarray*}
\mathcal{E}_{\eta }^{\ast } &\leq &\mathbb{E}_{\mu \sim \eta }\left[ \mathbb{
E}_{Z\sim \mu }\xi \big( \left\langle D_{M}\left( Bu_{\mu }\right)
,D_{M}X\right\rangle \text{sgn}\left( \left\langle u_{\mu },X\right\rangle
\right) \big) \right]  \\
&=&\mathbb{E}_{\mu \sim \eta }\left[ \mathbb{E}_{X\sim \sigma }\xi \left(
B\left\vert \left\langle u_{\mu },X\right\rangle \right\vert \right) \right] 
\\
&\leq &\sup_{\left\Vert u\right\Vert \leq 1}\mathbb{E}_{X\sim \sigma }\xi
\left( B\left\vert \left\langle u,X\right\rangle \right\vert \right) .
\end{eqnarray*}
For any unit vector $u\in H$ the density of the distribution of $\left\vert
\left\langle u,X\right\rangle \right\vert $ under $\sigma $ has maximum $
A_{d-1}/A_{d}$, where $A_{d}$ is the volume of $\mathcal{S}_{d}$ in the
metric inherited from $\mathbb{R}^{d}$. This density can therefore be
bounded by $\sqrt{d}/2$. Thus 
\begin{equation*}
\mathcal{E}_{\eta }^{\ast }\leq \sqrt{d}\int_{-\infty }^{\infty }\xi \left(
B\left\vert s\right\vert \right) ds=\frac{\sqrt{d}}{2B}=\epsilon ,
\end{equation*}
if we set $B=\sqrt{d}/\left( 2\epsilon \right) $. This choice is made to
ensure that the Lipschitz loss upper bounds the 0-1-loss.

Now let $\mathbf{\bar{Z}}$ be a multi-sample generated from the environment $
\eta $ and assume that we have solved the optimization problem (\ref{the
optimization problem}) to obtain the representation (or feature-map) $\hat{D}
\in \mathcal{H}$. Using the excess risk bound, Theorem \ref{Theorem LTL for
subspace}, and the fact that $\left\Vert C\right\Vert _{\infty }=1/d$ and $
\left\Vert C\right\Vert _{1}=1$, we get with probability at least $1-\delta $
in the draw of $\mathbf{\bar{Z}}$, that
\begin{eqnarray}
\mathcal{E}_{\eta }( \hat{D})  &\leq &\epsilon +\frac{\sqrt{2\pi }
}{2\epsilon }\left( K\sqrt{\frac{d}{T}}+\sqrt{\frac{K}{n}}\right) +\sqrt{
\frac{8\ln \left(4/\delta \right)}{T}}  \notag \\
&\leq &\sqrt{\sqrt{2\pi }\left( K\sqrt{\frac{d}{T}}+\sqrt{\frac{K}{n}}
\right) }+\sqrt{\frac{8\ln \left(4/\delta \right)}{T}},
\label{upper bound half-space learning}
\end{eqnarray}
if we optimize $\epsilon $. This guarantees the expected performance of
future uses of the representation $\hat{D}$. The high dimension still is a
hindrance to the estimation of the representation, but, as announced, its
effect vanishes in the limit $T\rightarrow \infty $. The individual samples
must only well outnumber the dimension $K$, roughly the number of shared
features.

We compare this upper bound to a lower bound for a large class of algorithms
which learn the tasks independently. 

\begin{definition}
An algorithm $f:\mathcal{S}_{d}\times \left\{ -1,1\right\} ^{n}\rightarrow 
\mathcal{S}_{d}$ is called \emph{orthogonally equivariant} if 
\begin{equation}
f\left( V\mathbf{x},\mathbf{y}\right) =Vf\left( \mathbf{x},\mathbf{y}\right), ~\text{for~every~orthogonal~matrix~}V\in {\mathbb{R}}^{d\times d}.
\end{equation}
\end{definition}

For data transformed by an orthogonal transformation an orthogonally
equivariant algorithm produces a correspondingly transformed hypothesis. Any
algorithm which does not depend on a specific coordinate system is
orthogonally equivariant. This class of algorithms includes all kernel
methods, but it excludes the Lasso (L1-norm regularization). If the known properties of the problem
posses a rotation symmetry only equivariant algorithms make sense.

Below we denote by $\mathrm{err}(u,v)$ the misclassification error between
the half-spaces associated with unit vectors $u$ and $v$, that is $\mathrm{
err}(u,v)=\Pr_{x\sim \sigma }\{\langle u,x\rangle \langle v,x\rangle <0\}$.
The following lower error bound is given in \citep{MP08}.

\begin{theorem}
\label{Theorem Main} Let $n<d$ and suppose that $f:\mathcal{S}_{d}^{n}\times
\left\{ -1,1\right\} ^{n}\rightarrow \mathcal{S}_{d}$ is an orthogonally
equivariant algorithm. Then for $\delta >0$ with probability at least $
1-\delta $ in the draw of $\mathbf{X\sim \sigma }^{n}$ we have for every $
u\in \mathcal{S}_{d}$ that
\begin{equation*}
\mathrm{err}\big(u,f(\mathbf{X},u(\mathbf{X}))\big)\geq \frac{1}{\pi }\left( 
\sqrt{\frac{d-n}{d}}-\sqrt{\frac{\ln \left( 1/\delta \right) }{d}}\right) ,
\end{equation*}
where $u\left( \mathbf{X}\right) =\left( \text{sgn}\left\langle
u,X_{1}\right\rangle ,\dots,\text{sgn}\left\langle u,X_{n}\right\rangle
\right) $.
\end{theorem}

If we use a union bound to subtract the upper bound (\ref{upper bound
half-space learning}) from this lower bound we obtain high probability
guarantees for the advantage of representation learning over other
algorithms. 


In the following section we plot the phase diagram derived here, namely the difference between the uniform lower bound and our upper bound, and compare it with empirical results (see Figure \ref{fig:LTL_results}).

\subsection{Numerical Experiments}

The purpose of the experiments is to compare MTL and LTL to independent task learning (ITL) in the simple setting of linear feature learning (or subspace learning)\footnote{The code used for the experiments presented in this section is available at \url{http://romera-paredes.com/multitask-representation}.}. We wish to study the regime in which MTL/LTL learning is beneficial over ITL as a function of the number of tasks $T$ and the sample size per task $n$. 

We consider noiseless linear binary classification tasks, namely halfspace learning. We generated the data in the following way. The ground truth weight vectors $u_1,\dots,u_T$ are obtained by the equation $u_t=Dc_t$, where $c_t\in\mathbb{R}^K$ is sampled from the uniform distribution on the unit sphere in $\R^K$, and the dictionary $D\in\mathbb{R}^{d\times K}$ is created by first sampling a $d$-dimension orthonormal matrix from the Haar measure, and then selecting the first $K$ columns (atoms).
We create all input marginals by sampling from the uniform distribution on the $\sqrt{d}$ radius sphere in $\R^d$. For each task we sample $n$ instances to build the training set, and $1000$ instances for the test set.

We train the methods with the hinge loss function $h(z):=\max\{0,1-z/c\}$, where $c$ is the margin. We choose $c=2/\epsilon$, so that the true error relative to the best hypothesis is of order $\epsilon$. We fixed the value of $\epsilon$ to be $(K/n)^{1/2}$.
For ITL we optimize that loss function constraining the $\ell_2$-norm of the weights, for MTL and LTL we constrain $D$ to have a Frobenius norm less or equal than $1$, and each $c_t$ is constrained to have an $\ell_2$ norm less or equal than $1$.
During testing we use the $0$-$1$ loss. For example the task-average error is evaluated as
\begin{equation}
\frac{1}{T} \sum_{t=1}^T \frac{1}{1000} \sum_{i=1}^{1000} 1\{{\rm sign}(\lb u_t,x_i\rb) \neq {\rm sign} (\lb {\hat u}_t,x_i\rb)\}\label{eq:testError}
\end{equation}
where ${\hat u}_t$ are the weight vectors learned by the assessed method.

\subsection{MTL Experiment}
We first discuss the MTL experiment. We let $d=50$, and vary $T \in \{5,10,\dots,150\}$, $n \in \{5,10,\dots,150\}$ considering the cases $K=2$ and $K=5$.
In Figure \ref{fig:MTL_results} we report the difference between the classification error of the two methods. These results are obtained by repeating the experiment $10$ times, reporting the average difference. In each trial a different set of input points and underlying weight vectors are generated for each task. 
In the MTL case the training error was always below $0.1$ and on
average it was smaller than $0.04$. This suggests that despite the problem being non-convex, the gradient optimization algorithm finds a good suboptimal solution.

\begin{figure}[t]
\centering
\includegraphics[width=0.45\linewidth]{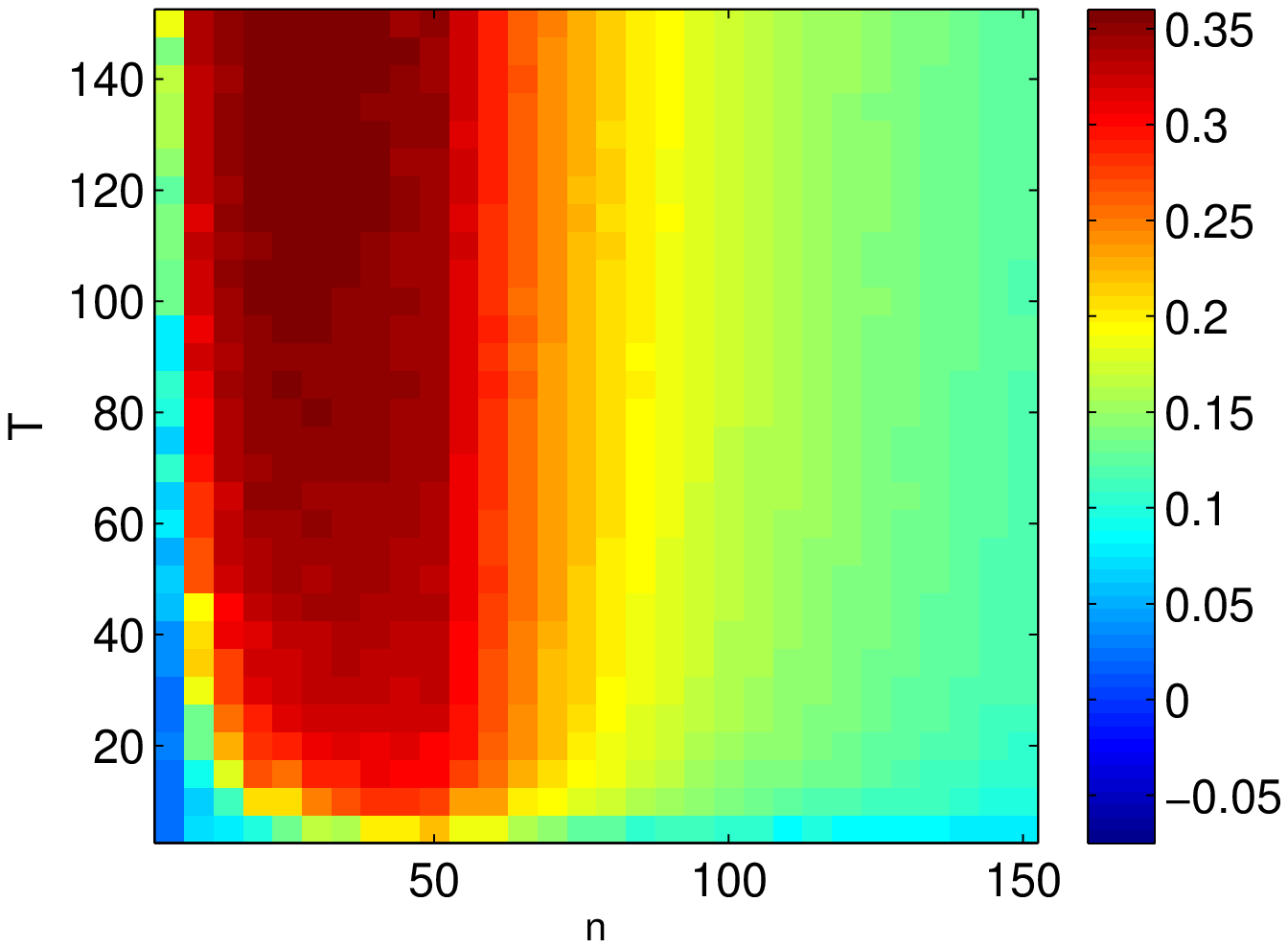}
\includegraphics[width=0.45\linewidth]{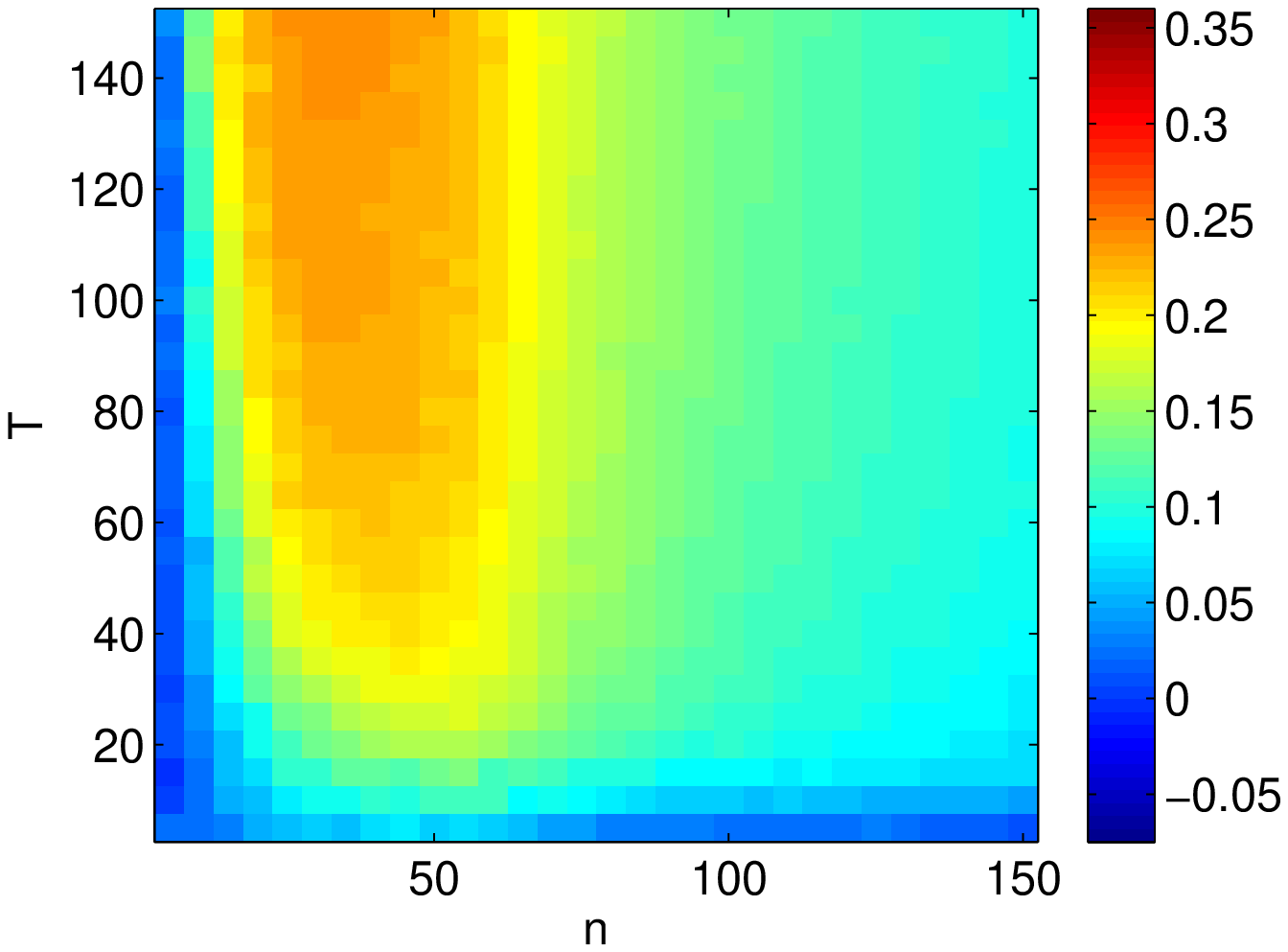}
\caption{Difference of test classification error, computed according to eq. (\ref{eq:testError}), between ITL and MTL. The vertical axis represents the number of training tasks, and the horizontal axis the number of training instances per task.
In the left column $K=2$, and in the right column $K=5$.}
\label{fig:MTL_results}
\end{figure}

We have made further experiments to assess the influence of other data settings on the difference between ITL and MTL. In the first of those experiments we have explored the cases in which the dictionary size is overestimated and underestimated. The results are shown in Figure \ref{fig:Different_K}. In the left plot the dictionary size is overestimated, in particular the ground truth number of atoms is $2$, and the number of atoms used in the MTL method is $5$. We can appreciate a similar pattern as the one we saw in Figure \ref{fig:MTL_results}, although differences between ITL and MTL are not as high. The performance is slightly hampered, as expected due to an overestimation of the number of atoms. On the other hand in Figure \ref{fig:Different_K} (right) we show the results when the number of atoms in the ground truth dictionary is $5$, whereas the number of atoms used in the MTL approach is $2$. In this case we see that the performance is severely affected by the underestimation of the size of the dictionary, yet we observe that MTL performs better than ITL in the same regime as in the previous experiments.

\begin{figure}[t]
\centering
\includegraphics[width=0.45\linewidth]{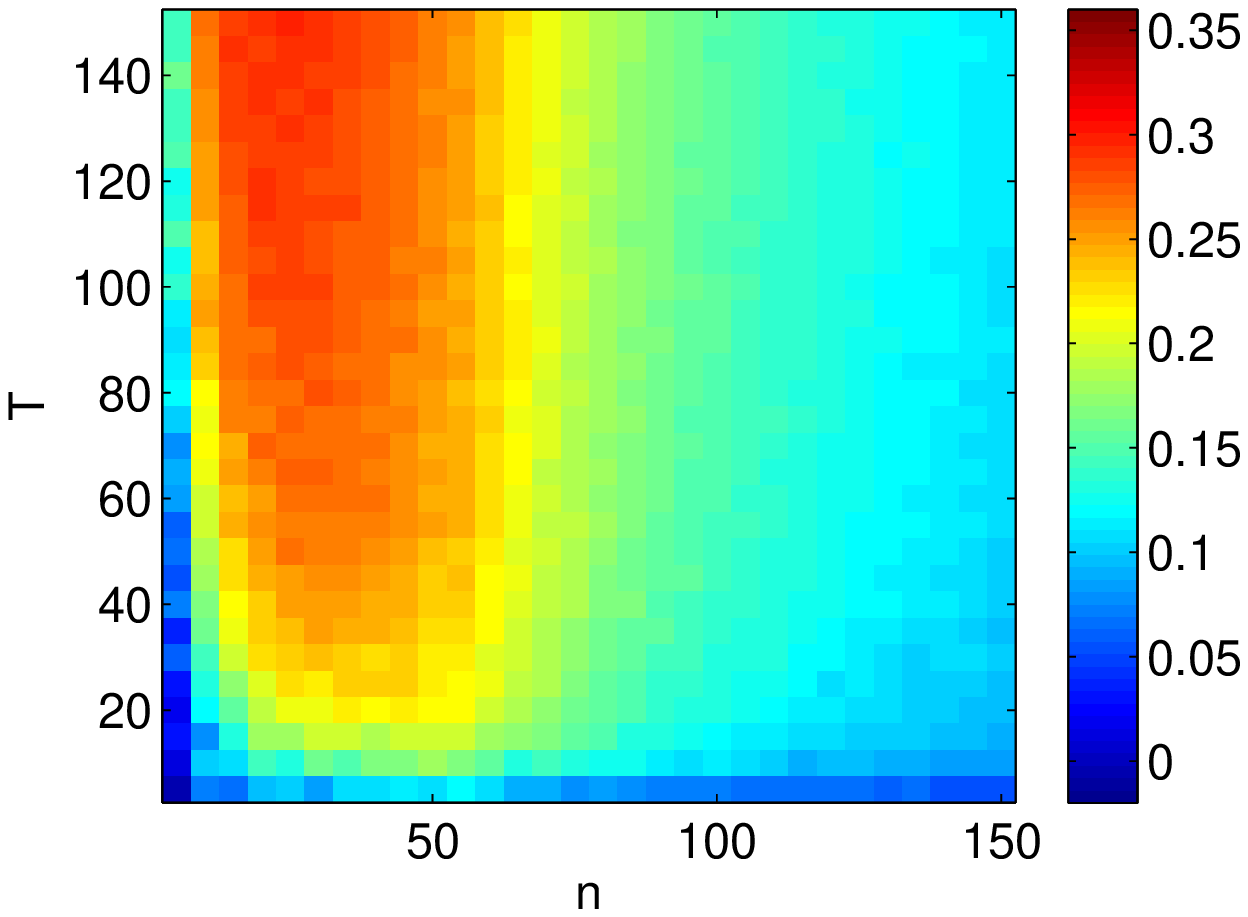}
\includegraphics[width=0.45\linewidth]{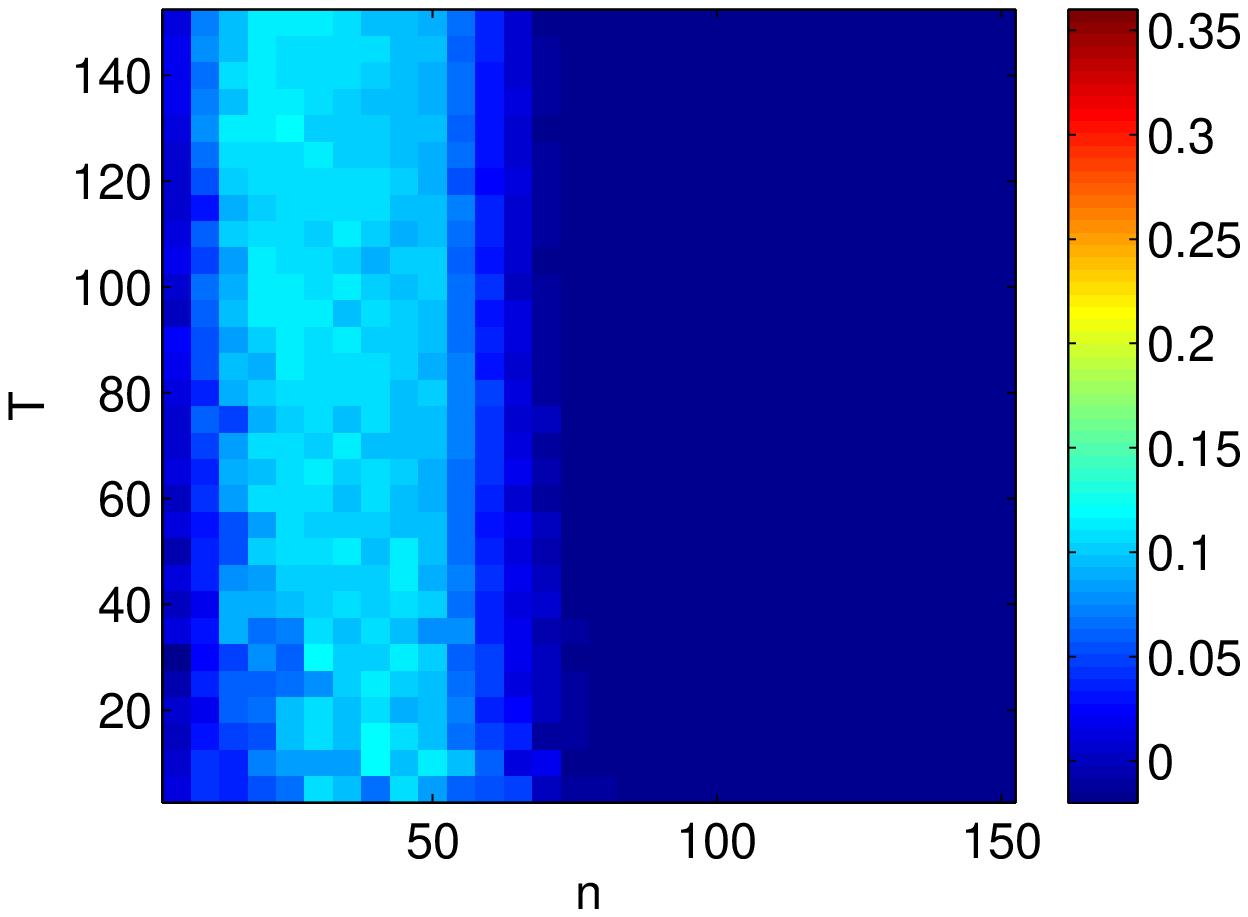}
\caption{Difference of test classification error, computed according to eq. (\ref{eq:testError}), between ITL and MTL, when the number of atoms of the ground truth dictionary does not match the number of atoms of the MTL model. The plot in the left shows the experiment in which the ground truth number of atoms is $2$, whereas the number of atoms used in the MTL approach is $5$. The plot in the right shows the opposite scenario: $5$ atoms as ground truth, and $2$ atoms in the MTL model. The vertical axis represents the number of training tasks, and the horizontal axis the number of training instances per task.}
\label{fig:Different_K}
\end{figure}

In the second of these experiments we study how the results are affected when the data are noisy. To do so, we have generated the data so that the ground truth label for instance $x_i$ for task $t$ is given by ${\rm sign}(\lb u_t,x_i\rb + \varepsilon_{ti})$, where $\varepsilon_{ti}\sim\mathcal{N}(0,1)$. The dictionary size, for both the ground truth and the MTL approach, is $K=2$. The results are shown in Figure \ref{fig:noise}, and we can see a similar behaviour as the one in Figure \ref{fig:MTL_results}, with somewhat smaller differences between ITL and MTL.

\begin{figure}[t]
\centering
\includegraphics[width=0.45\linewidth]{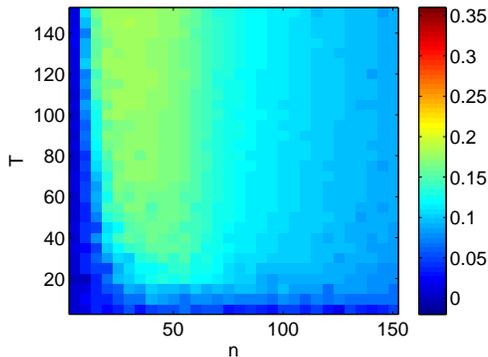}
\caption{Difference of test classification error, computed according to eq. (\ref{eq:testError}), between ITL and MTL, when adding Gaussian noise to the ground truth labels. The vertical axis represents the number of training tasks, and the horizontal axis the number of training instances per task.}
\label{fig:noise}
\end{figure}


\subsection{LTL Experiment}

In this experiment we test how the dictionary learned at the training stage helps learning new tasks, and we assess how similar the resultant figure is in comparison to the phase diagram derived in the previous section. 

The data is generated according to the settings given in the MTL experiment. Furthermore, $50$ new tasks are sampled following the same scheme previously described for the purpose of computing the LTL test error.  
We present the results in Figure \ref{fig:LTL_results} (Top). Similar to the previous experiment, we report the average difference between the test error of ITL and LTL after $10$ trials.

In Figure \ref{fig:LTL_results} (Bottom) we present the theoretical phase diagram, which was generated using $1\leq T\leq 10^{11}$, $1\leq n\leq 10^{5}$, $d=10^{5}$, $\protect\delta =0.0001$. We also plot as a dark line the points in which there is no difference in the performances between ITL and LTL.

\begin{figure}[t]
\centering
\includegraphics[width=0.45\linewidth]{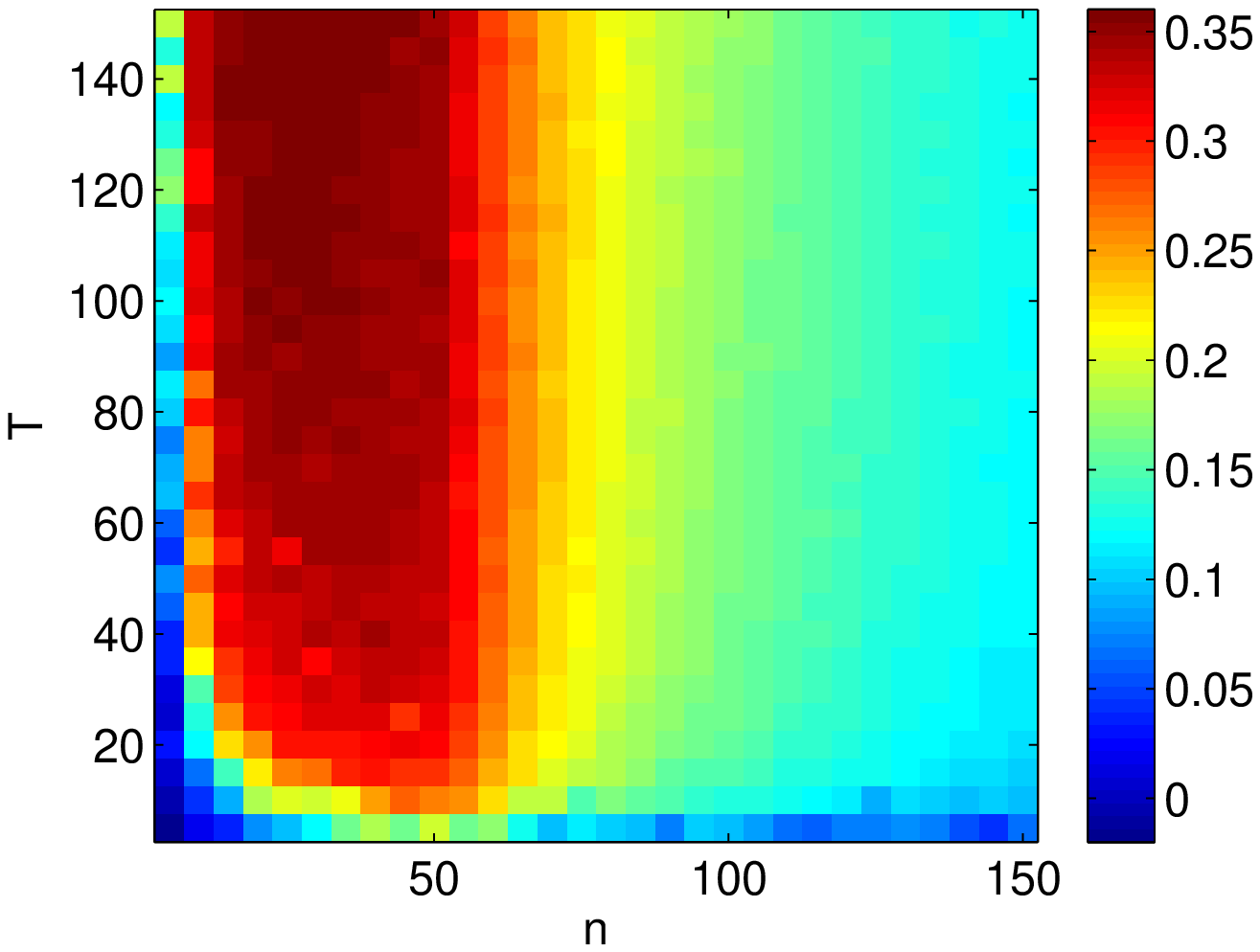}
\includegraphics[width=0.45\linewidth]{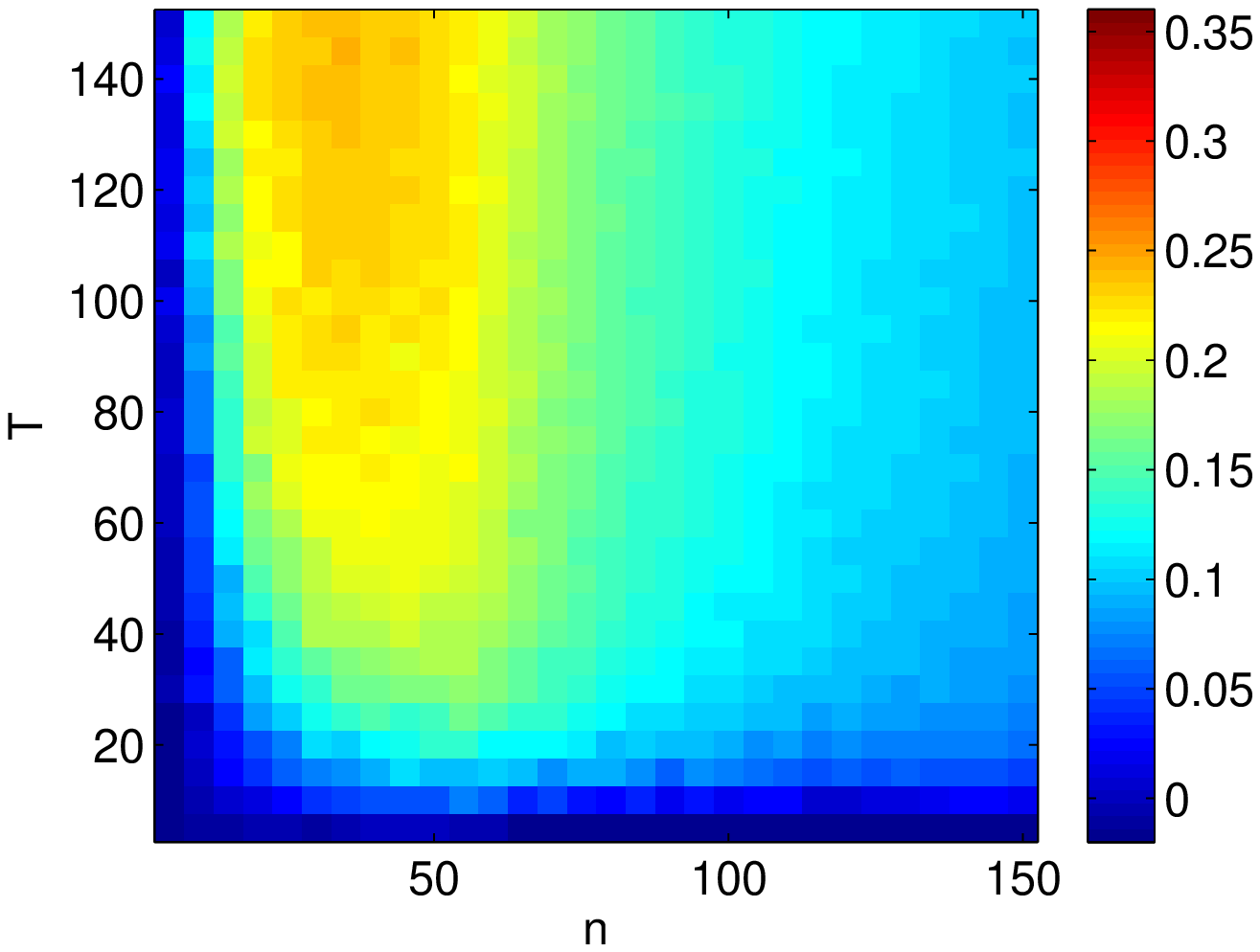}
\includegraphics[width=0.45\linewidth]{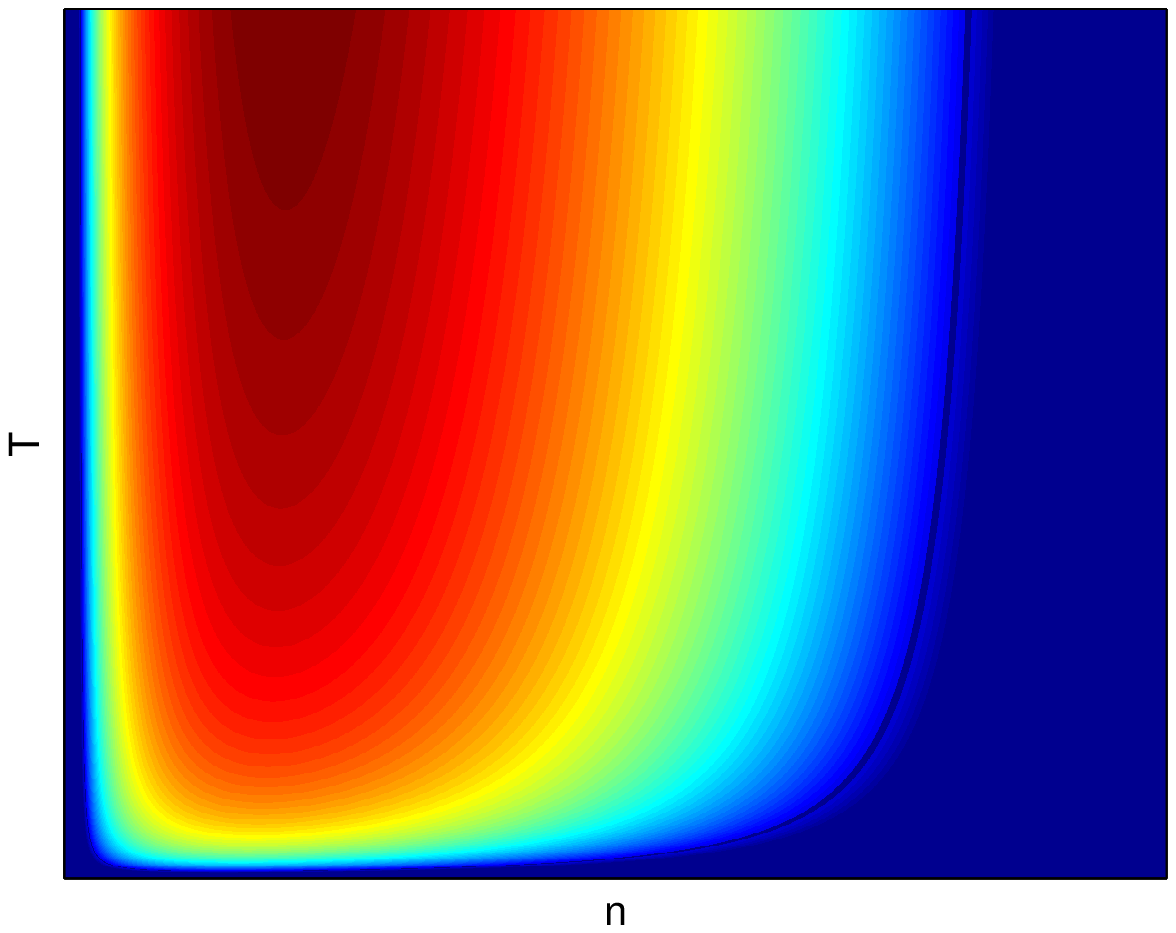}
\includegraphics[width=0.45\linewidth]{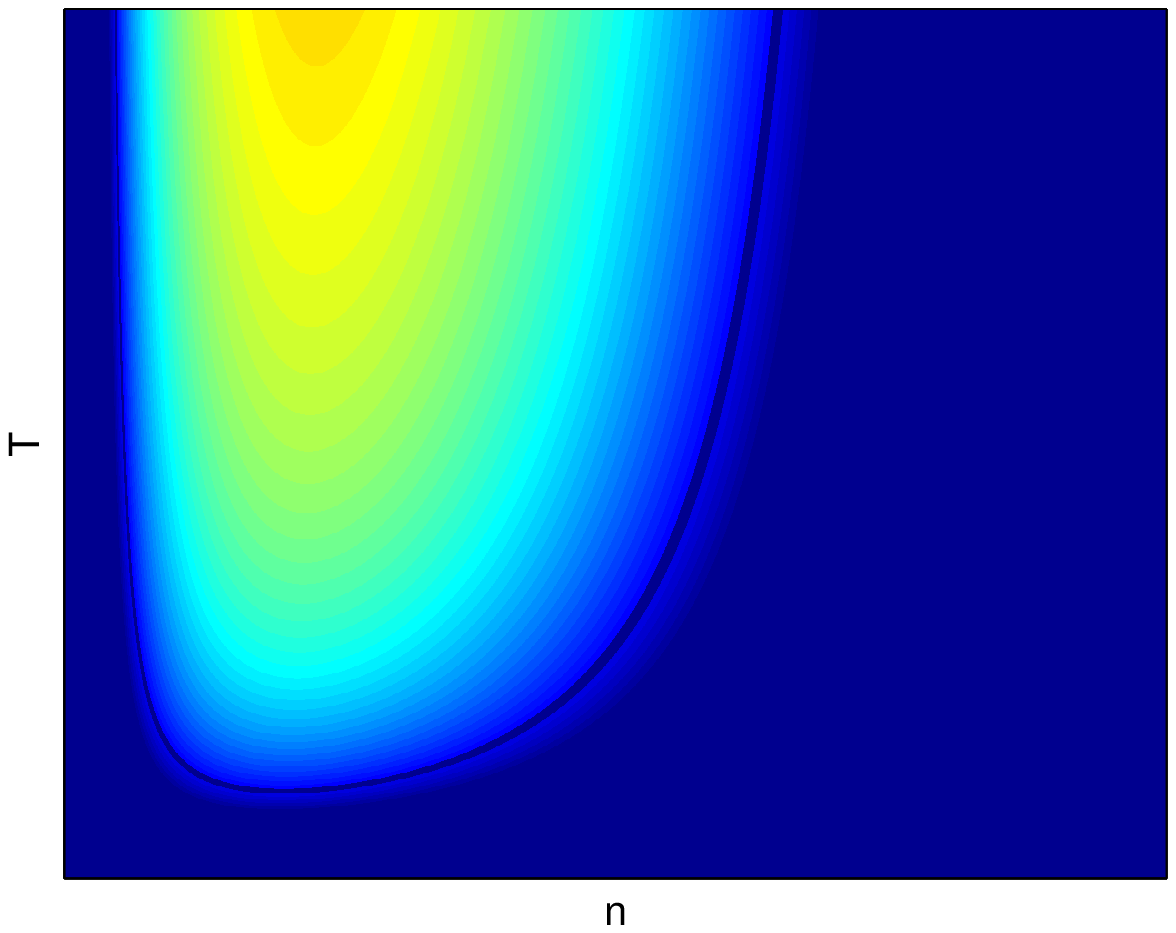}
\caption{The vertical axis represents the number of training tasks, and the horizontal axis the number of training instances per task. 
Plots in the top row show the difference of test classification error, computed on $50$ new tasks, between ITL and LTL. Plots in the bottom row show the region where the upper bound for LTL is smaller than the lower bound
for any equivariant algorithm for ITL (see the discussion in Section 3.1, in particular Equation \ref{upper bound half-space learning}) using $1\leq T\leq 10^{11}$, $1\leq n\leq 10^{5}$
, $d=10^{5}$, and $\protect\delta =0.0001$. 
 In the left column $K=2$, and in the right column $K=5$.}
\label{fig:LTL_results}
\end{figure}

The reader may object about the much larger parameter values used to generate the plots of theoretical differences, in comparison to the experimental settings. These large parameters are partly a consequence of an accumulation of somewhat loose
estimates in the derivation of both the upper and lower bounds. Another
reason is that in applying it to a noiseless, finite-dimensional problem
(for clarity) we have sacrificed two strong points of our results:
independence of input dimension and its agnostic nature. Apart from the large parameter values the theoretical prediction shown in
Figure \ref{fig:LTL_results} (Bottom) is in very good agreement with the
experimental results in Figure \ref{fig:LTL_results} (Top).

We have also performed experiments in order to evaluate the influence of noise and under/overestimation of the dictionary size on the difference between ITL and LTL. We obtained similar results as the ones reported for MTL in Figures \ref{fig:Different_K} and \ref{fig:noise}.

Finally, we have compared the learned dictionary, $\hat{D}$, with the ground truth, $D$, in the same regime of parameters used for the previous experiments. Note that a dictionary could be correct up to permutations and changes of sign of its atoms. To overcome this issue we use the similarity measure 
\beq
s(\hat{D},D)=\frac{1}{K}\left\Vert D^{\top}\hat{D}\right\Vert _{{\rm tr}},
\label{eq:sss}
\eeq
where $\|\cdot\|_{\rm tr}$ is the sum of singular values of a matrix. Note that $s(\hat{D},D)=1$ if $\hat{D}$ and $D$ are the same matrix up to permutation of columns and changes of sign, as requested. The results are found in Figure \ref{fig:DictionaryDiscovery}. 

Figure \ref{fig:DictionaryDiscovery} indicate that the learned dictionary is close to the true dictionary even for small sample sizes, provide $T$ is large. This supports the results in Figure \ref{fig:MTL_results} and the top plots in Figure \ref{fig:LTL_results}, where MTL or LTL are found to be superior to ITL in this regime, respectively. 

\begin{figure}[t]
\centering
\includegraphics[width=0.45\linewidth]{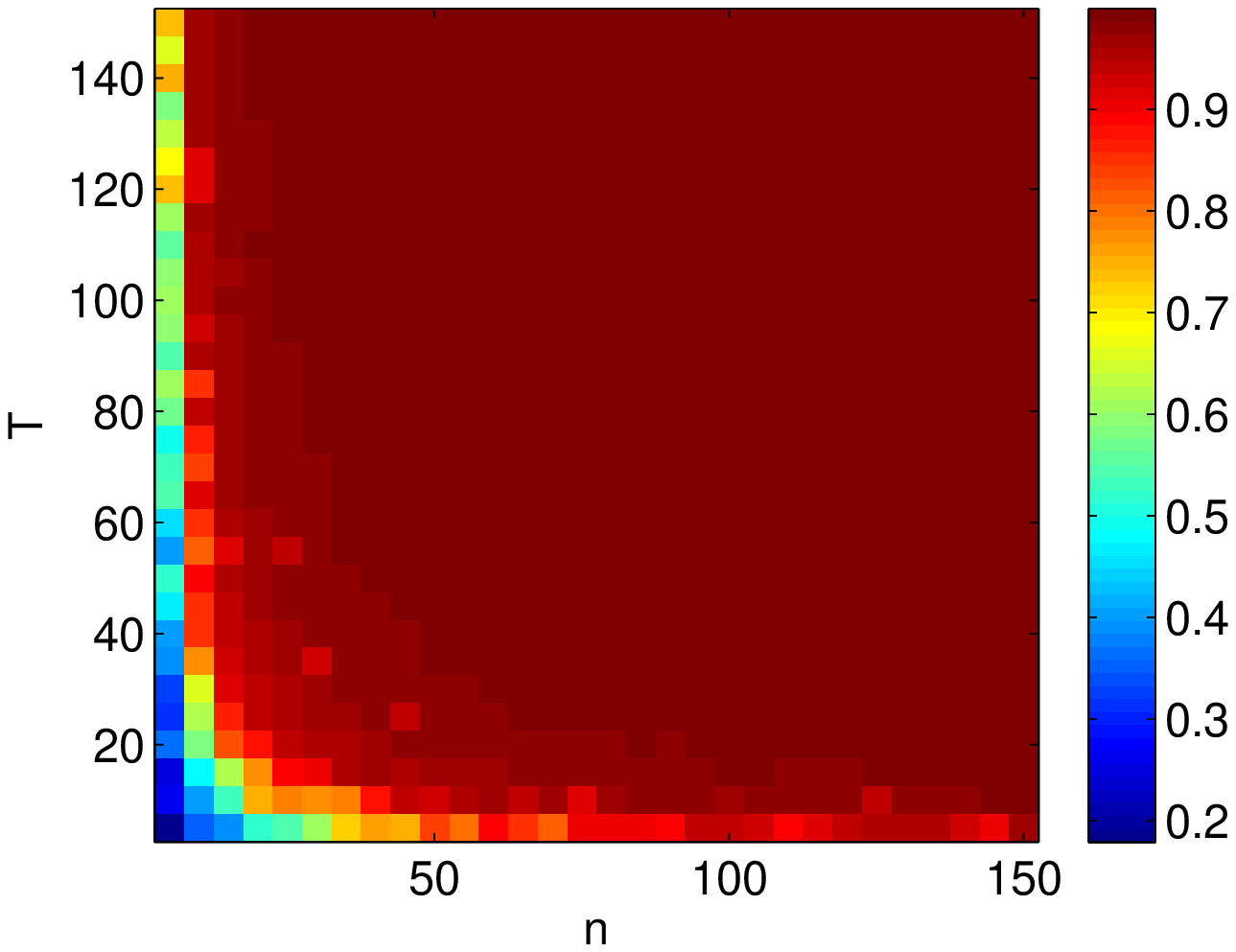}
\includegraphics[width=0.45\linewidth]{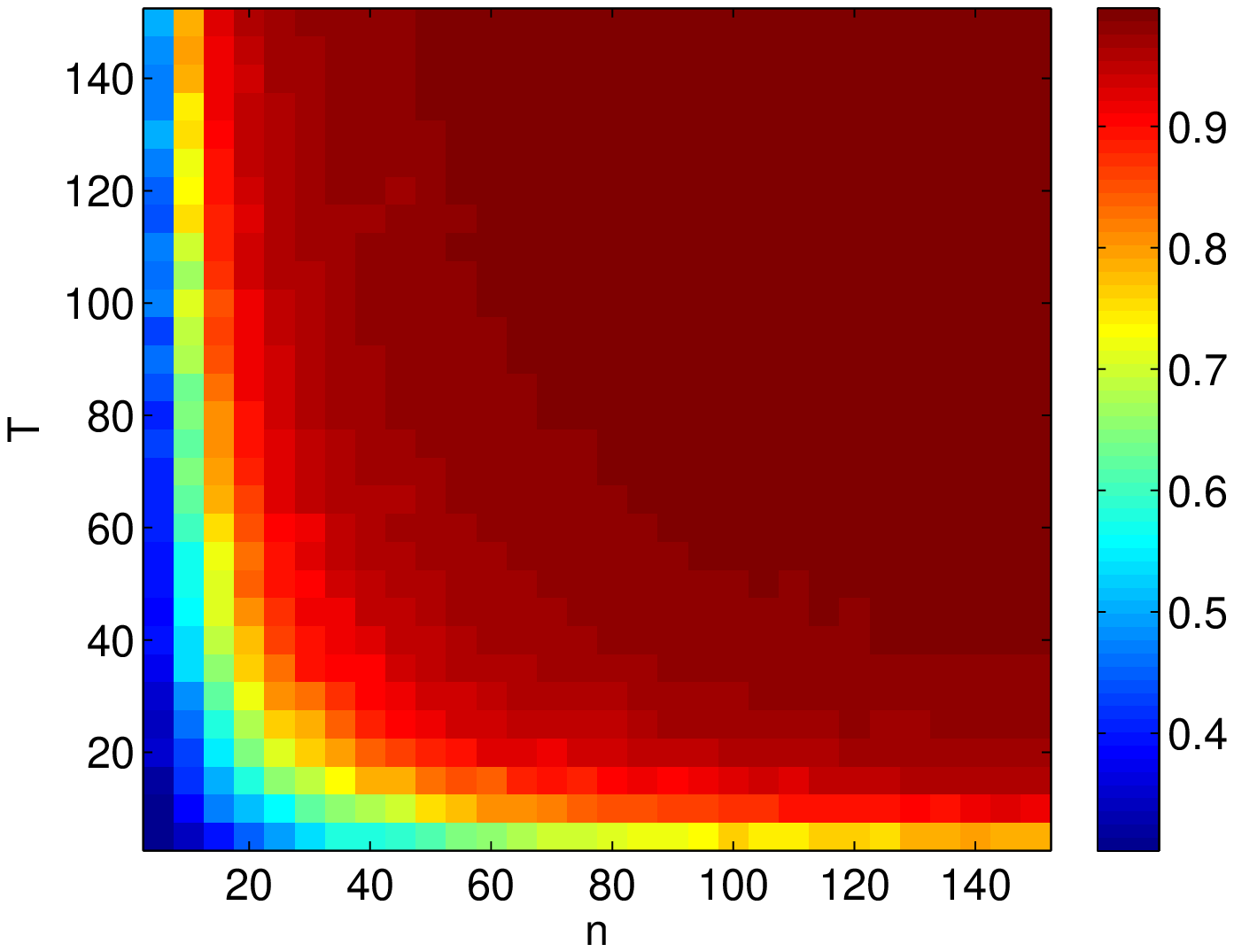}
\caption{Similarity between the learned dictionary $\hat{D}$ and the ground truth dictionary $D$, according the similarity measure $s(\hat{D},D)$ in Equation \eqref{eq:sss}. The vertical axis represents the number of training tasks, and the horizontal axis the number of training instances per task. Left plot: $K=2$. Right plot: $K=5$.}
\label{fig:DictionaryDiscovery}
\end{figure}

\section{Proofs of the Main Theorems}
\label{sec:proofs}
In this section we prove our principal results, Theorem \ref{thm:MTL} and
Theorem \ref{thm:LTL}. In preparation for the proofs we will first present
some important auxiliary results.

\subsection{Tools}

We denote by $\gamma $ a generic vector of independent standard normal
variables, whose dimension will be clear from context. A central role in
this paper is played by the {Gaussian average} $G(Y)$ of a set $Y\subseteq 
\mathbb{R}^{n}$, which is defined as 
\begin{equation*}
G\left( Y\right) =\mathbb{E}\sup_{y\in Y}\left\langle \gamma ,y\right\rangle
=\mathbb{E}\sup_{y\in Y}\sum_{i=1}^{n}\gamma _{i}y_{i}.
\end{equation*}%
The reader who is concerned about the measurability of the random variable
on the right hand side should replace $Y$ by a countable dense subset of $Y$%
, with similar adjustments wherever the Gaussian averages occur.

Rademacher averages, where the $\gamma _{i}$ are replaced by uniform $%
\left\{ -1,1\right\} $-distributed variables, are somewhat more popular in
the literature. We use Gaussian averages instead, because in most cases they
are just as easy to bound and possess special properties (Theorem \ref%
{Slepian Lemma} and Theorem \ref{Theorem Chain Rule} below) which we need in
our analysis$.$

The first result is a standard tool to prove uniform bounds on the
estimation error in terms of Gaussian averages \citep{Bartlett 2002}.

\begin{theorem}
\label{Theorem Generalization}Let $\mathcal{\tciFourier }$ be a real-valued
function class on a space $\mathcal{X}$ and let $\mathbf{X}=\left(
X_{1},...,X_{n}\right) $ be a vector of independent random variables and $%
\mathbf{X}^{\prime }$ iid to $\mathbf{X}$. Then

(i)%
\begin{equation*}
\mathbb{E}_{\mathbf{X}}\sup_{f\in \mathcal{F}}\frac{1}{n}\sum_{i=1}^{n}
\left( \mathbb{E}_{\mathbf{X}^{\prime }}\left[ f\left( X_{i}^{\prime
}\right) \right] -f\left( X_{i}\right) \right) \leq \frac{\sqrt{2\pi } 
\mathbb{E}_{\mathbf{X}}G\left( \mathcal{F}\left( \mathbf{X}\right) \right) }{
n}
\end{equation*}

(ii) if the members of $\tciFourier $ have values in $\left[ 0,1\right] $
then with probability greater than $1-\delta $ in $\mathbf{X}$ for all $f\in 
\mathcal{\tciFourier }$ 
\begin{equation*}
\frac{1}{n}\sum_{i=1}^{n}\left( \mathbb{E}_{\mathbf{X}^{\prime }}\left[
f\left( X_{i}^{\prime }\right) \right] -f\left( X_{i}\right) \right) \leq 
\frac{\sqrt{2\pi }G\left( \mathcal{F}\left( \mathbf{X}\right) \right) }{n}+%
\sqrt{\frac{9\ln \left( 2/\delta \right) }{2n}}.
\end{equation*}
\end{theorem}

The following theorem is a vector-valued version of the above, is useful for
bounds on the task-averaged estimation error (\cite{Zhang 2005}, \cite%
{Maurer 2006}).

\begin{theorem}
\label{Theorem Generalization Multitask}Let $\mathcal{\tciFourier }$ be a
class of functions $f:\mathcal{X\rightarrow }\left[ 0,1\right] ^{T}$, and
let $\mu _{1},...,\mu _{T}$ be probability measures on $\mathcal{X}$ with $%
\mathbf{\bar{X}=}\left( \mathbf{X}_{1},...,\mathbf{X}_{T}\right) \sim
\prod_{t=1}^{T}\left( \mu _{t}\right) ^{n}$ where $\mathbf{X}_{t}=\left(
X_{t1},...,X_{tn}\right) $. Then with probability greater than $1-\delta $
in $\mathbf{\bar{X}}$ for all $f\in \mathcal{\tciFourier }$ 
\begin{equation*}
\frac{1}{T}\sum_{t}\left( \mathbb{E}_{X\sim \mu _{t}}\left[ f_{t}\left(
X\right) \right] -\frac{1}{n}\sum_{ti}f_{t}\left( X_{ti}\right) \right) \leq 
\frac{\sqrt{2\pi }G\left( Y\right) }{nT}+\sqrt{\frac{9\ln \left( 2/\delta
\right) }{2nT}},
\end{equation*}%
where $Y\subset \mathbb{R}^{Tn}$ is the random set defined by $Y=\left\{
\left( f_{t}\left( X_{ti}\right) \right) :f\in \tciFourier \right\} .$%
\bigskip 
\end{theorem}

The previous two theorems replace the problem of proving uniform bounds by
the problem of bounding Gaussian averages. One key result in the latter
direction is known as Slepian's Lemma (\cite{Slepian 1962}, \cite{Ledoux
Talagrand 1991}).

\begin{theorem}
\label{Slepian Lemma}Let $\Omega $ and $\Xi $ be mean zero, separable
Gaussian processes indexed by a common set $\mathcal{S}$, such that 
\begin{equation*}
\mathbb{E}\left( \Omega _{s_{1}}-\Omega _{s_{2}}\right) ^{2}\leq \mathbb{E}
\left( \Xi _{s_{1}}-\Xi _{s_{2}}\right) ^{2}\text{ for all }s_{1},s_{2}\in 
\mathcal{S}\text{.}
\end{equation*}%
Then%
\begin{equation*}
\mathbb{E}\sup_{s\in \mathcal{S}}\Omega _{s}\leq \mathbb{E}\sup_{s\in 
\mathcal{S}}\Xi _{s}.
\end{equation*}
\end{theorem}

The following corollary is the key to our bound for LTL.

\begin{corollary}
\label{Contraction Lemma} Let $Y\subseteq \mathbb{R}^{n}$ and let $\phi
:Y\rightarrow \mathbb{R}^{m}$ be (Euclidean) Lipschitz with Lipschitz
constant $L$. Then 
\begin{equation*}
G\left( \phi \left( Y\right) \right) \leq LG\left( Y\right).
\end{equation*}
\end{corollary}

\begin{proof}
Define two Gaussian processes indexed by $Y$ as%
\begin{equation*}
\Omega _{y}=\sum_{k=1}^{m}\gamma _{k}\phi \left( y\right) _{k}\text{ and }%
\Xi _{y}=L\sum_{i=1}^{n}\gamma _{i}^{\prime }y_{i}\text{,}
\end{equation*}%
with independent $\gamma _{k}$ and $\gamma _{i}^{\prime }$. Then for any $%
y,y^{\prime }\in Y$%
\begin{equation*}
\mathbb{E}\left( \Omega _{y}-\Omega _{y^{\prime }}\right) ^{2}=\left\Vert
\phi \left( y\right) -\phi \left( y^{\prime }\right) \right\Vert ^{2}\leq
L^{2}\left\Vert y-y^{\prime }\right\Vert ^{2}=\mathbb{E}\left( \Xi
_{s_{1}}-\Xi _{s_{2}}\right) ^{2}\text{,}
\end{equation*}%
so that, by Slepian's Lemma,%
\begin{equation*}
G\left( \phi \left( Y\right) \right) =\mathbb{E}\sup_{y\in Y}\Omega _{y}\leq 
\mathbb{E}\sup_{y\in Y}\Xi _{y}=LG\left( Y\right) .
\end{equation*}
\end{proof}

In many applications this is applied when $n=m$ and $\phi $ is defined by $%
\phi \left( y_{1},...,y_{n}\right) =\left( \phi _{1}\left( y_{1}\right)
,...,\phi _{n}\left( y_{n}\right) \right) $ where the real functions $\phi
_{1},...,\phi _{n}$ have Lipschitz constant $L$.

At one point we will need a generalization of the above corollary, which
allows to select $\phi $ from an entire class of Lipschitz functions. We
will use the following result, which is taken from \cite{Maurer 2014}.  It
will play an important role in the proof of Theorem \ref{thm:unifMTL} below.

\begin{theorem}
\label{Theorem Chain Rule}Let $Y\subseteq \mathbb{R}^{n}$ have (Euclidean)
diameter $D\left( Y\right) $ and let $\tciFourier $ be a class of functions $%
f:Y\rightarrow \mathbb{R}^{m}$, all of which have Lipschitz constant at most 
$L\left( \tciFourier \right) $. Then for any \thinspace $y_{0}\in Y$ 
\begin{equation*}
G\left( \tciFourier \left( Y\right) \right) \leq c_{1}L\left( \tciFourier
\right) G\left( Y\right) +c_{2}D\left( Y\right) Q\left( \tciFourier \right)
+G\left( \tciFourier \left( y_{0}\right) \right) ,
\end{equation*}%
where $c_{1}$ and $c_{2}$ are universal constants and 
\begin{equation*}
Q\left( \tciFourier \right) =\sup_{\mathbf{y},\mathbf{y}^{\prime }\in Y,~%
\mathbf{y}\neq \mathbf{y}^{\prime }}\mathbb{E}\sup_{f\in \tciFourier }\frac{%
\left\langle \mathbf{\gamma },f\left( \mathbf{y}\right) -f\left( \mathbf{y}%
^{\prime }\right) \right\rangle }{\left\Vert \mathbf{y}-\mathbf{y}^{\prime
}\right\Vert }.
\end{equation*}
\end{theorem}

Note that the result allows us to minimize the right hand side in $y_{0}$.
Analogs of Theorem \ref{Slepian Lemma} and Theorem \ref{Theorem Chain Rule}
are not available for Rademacher averages. This is the reason why we use the
slightly more exotic Gaussian averages.

\subsection{Proof of the Excess Risk Bound for the Average Risk}

We first establish the following uniform bound. It is of some interest in
its own right, in particular since the problem (\ref{the optimization
problem}) is often non-convex, so that the excess risk bound may not be
meaningful in practice. Recall the definition of $Q$ given in Equation \eqref{eq:Q}.

\begin{theorem}
\label{Theorem uniform average risk}Let $\mu _{1},\dots ,\mu _{T}$ be
probability measures on $\mathcal{Z}$ and let $Z_{t1},\dots ,Z_{tn}$ be
i.i.d. from $\mu _{t}$, for $t=1,\dots ,T$. Let $\delta \in (0,1)$. With
probability at least $1-\delta $ in the draw of a multisample $\mathbf{\bar{Z%
}}$, it holds for every $h\in \mathcal{H}$ and every $f_{1},\dots ,f_{T}\in 
\mathcal{F}$ that 
\begin{multline*}
\mathcal{E}_{\mathrm{avg}}\left( h,f_{1},...,f_{T}\right) -\frac{1}{Tn}%
\sum_{ti}\ell \left( f_{t}\left( h\left( X_{ti}\right) \right)
,Y_{ti}\right)  \\
\leq c_{1}\frac{LG(\mathcal{H}(\mathbf{\bar{X}}))}{nT}+c_{2}\frac{%
Q\sup_{h\in \mathcal{H}}\left\Vert h\left( \mathbf{\bar{X}}\right)
\right\Vert }{n\sqrt{T}}+\sqrt{\frac{9\ln \left( 2/\delta \right) }{2nT}},
\end{multline*}%
where $c_{1}$ and $c_{2}$ are universal constants\label{thm:unifMTL}.
\end{theorem}

\begin{proof}
By Theorem \ref{Theorem Generalization Multitask},
with probability at least $1-\delta $ in $\mathbf{\bar{Z},}$ for all $h\in 
\mathcal{H}$ and all $f_{1},...,f_{T}\in \tciFourier $, we have that 
\begin{equation}
\mathcal{E}_{\mathrm{avg}}\left( h,f_{1},...,f_{T}\right) -\frac{1}{Tn}%
\sum_{ti}\ell \left( f_{t}\left( h\left( X_{ti}\right) \right)
,Y_{ti}\right) \leq \frac{\sqrt{2\pi }}{nT}G\left( S\right) +\sqrt{\frac{%
9\ln \left( 2/\delta \right) }{2nT}},  \label{eq:deviationMTL}
\end{equation}%
where $S=\left\{ \left( \ell \left( f_{t}\left( h\left( X_{ti}\right)
\right) ,Y_{ti}\right) \right) :f\in \tciFourier ^{T}\text{ and }h\in 
\mathcal{H}\right\} \subseteq \mathbb{R}^{Tn}$. By the Lipschitz property of
the loss function $\ell $ and the contraction lemma Corollary \ref%
{Contraction Lemma} (recall the remark which follows its proof) we have $%
G\left( S\right) \leq G\left( S^{\prime }\right) $, where $S^{\prime
}=\left\{ \left( f_{t}\left( h\left( X_{ti}\right) \right) \right) :f\in
\tciFourier ^{T}\text{ and }h\in \mathcal{H}\right\} \subseteq \mathbb{R}%
^{Tn}$.

Recall that $\mathcal{H}\left( \mathbf{\bar{X}}\right) \subseteq \mathbb{R}%
^{KTn}$ is defined by%
\begin{equation*}
\mathcal{H}\left( \mathbf{\bar{X}}\right) =\left\{ \left( h_{k}\left(
X_{ti}\right) \right) :h\in \mathcal{H}\right\} ,
\end{equation*}%
and define a class of functions $\tciFourier ^{\prime }:\mathbb{R}%
^{KTn}\rightarrow \mathbb{R}^{Tn}$ by 
\begin{equation*}
\tciFourier ^{\prime }=\left\{ y\in \mathbb{R}^{KTn}\mapsto \left(
f_{t}\left( y_{ti}\right) \right) :\left( f_{1},...,f_{T}\right) \in
\tciFourier ^{T}\right\} .
\end{equation*}%
Then $S^{\prime }=\tciFourier ^{\prime }\left( \mathcal{H}\left( \mathbf{%
\bar{X}}\right) \right) $, and by Theorem \ref{Theorem Chain Rule} for
universal constants $c_{1}^{\prime }$ and $c_{2}^{\prime }$%
\begin{equation}
G\left( S^{\prime }\right) \leq c_{1}^{\prime }L\left( \tciFourier ^{\prime
}\right) G\left( \mathcal{H}\left( \mathbf{\bar{X}}\right) \right)
+c_{2}^{\prime }D\left( \mathcal{H}\left( \mathbf{\bar{X}}\right) \right)
Q\left( \tciFourier ^{\prime }\right) +\min_{y\in Y}G\left( \tciFourier
\left( y\right) \right) .  \label{Application chain rule}
\end{equation}%
We now proceed by bounding the individual terms in the right hand side
above. Let $y,y^{\prime }\in \mathbb{R}^{KTn}$, where $y=\left(
y_{ti}\right) $ with $y_{ti}\in \mathbb{R}^{K}$ and $y^{\prime }=\left(
y_{ti}^{\prime }\right) $ with $y_{ti}^{\prime }\in \mathbb{R}^{K}$. Then
for $f=\left( f_{1},...,f_{T}\right) \in \tciFourier ^{T}$ 
\begin{eqnarray*}
\left\Vert f\left( y\right) -f\left( y^{\prime }\right) \right\Vert ^{2}
&=&\sum_{ti}\left( f_{t}\left( y_{ti}\right) -f_{t}\left( y_{ti}^{\prime
}\right) \right) ^{2} \\
&\leq &L^{2}\sum_{ti}\left\Vert y_{ti}-y_{ti}^{\prime }\right\Vert
^{2}=L^{2}\left\Vert y-y^{\prime }\right\Vert ^{2},
\end{eqnarray*}%
so that $L\left( \tciFourier ^{\prime }\right) \leq L$. Also 
\begin{eqnarray*}
&&\mathbb{E}\sup_{g\in \mathcal{\tciFourier }^{\prime }}\left\langle \gamma
,g\left( y\right) -g\left( y^{\prime }\right) \right\rangle  \\
&=&\mathbb{E}\sup_{\left( f_{1},...,f_{T}\right) \in \tciFourier
^{T}}\sum_{ti}\gamma _{ti}\left( f_{t}\left( y_{ti}\right) -f_{t}\left(
y_{ti}^{\prime }\right) \right)  \\
&=&\sum_{t}\mathbb{E}\sup_{f\in \tciFourier }\sum_{i}\gamma _{i}\left(
f\left( y_{ti}\right) -f\left( y_{ti}^{\prime }\right) \right)  \\
&\leq &\sqrt{T}\left( \sum_{t}\left( \mathbb{E}\sup_{f\in \tciFourier
}\sum_{i}\gamma _{i}\left( f\left( y_{ti}\right) -f\left( y_{ti}^{\prime }\right)
\right) \right) ^{2}\right) ^{1/2} \\
&\leq &\sqrt{T}\left( \sum_{t}Q^{2}\sum_{i}\left\Vert
y_{ti}-y_{ti}^{\prime }\right\Vert ^{2}\right) ^{1/2} \\
&=&\sqrt{T}Q\left\Vert y-y^{\prime }\right\Vert ,
\end{eqnarray*}%
whence $Q\left( \tciFourier ^{\prime }\right) =\sqrt{T}Q$. Finally we take $%
y_{0}=0$ and the last term in (\ref{Application chain rule}) vanishes since $%
f\left( 0\right) =0$ for all $f\in \tciFourier $. Substitution in (\ref%
{Application chain rule}) and using $G\left( S\right) \leq G\left( S^{\prime
}\right) $ we arrive at 
\begin{equation*}
G\left( S\right) \leq c_{1}^{\prime }LG\left( \mathcal{H}\left( \mathbf{\bar{%
X}}\right) \right) +c_{2}^{\prime }\sqrt{T}D\left( \mathcal{H}\left( \mathbf{%
\bar{X}}\right) \right) Q.
\end{equation*}%
Bounding $D\left( \mathcal{H}\left( \mathbf{\bar{X}}\right) \right) \leq
2\sup_{h}\left\Vert h\left( \mathbf{\bar{X}}\right) \right\Vert $ and
substitution in \eqref{eq:deviationMTL} gives the result.
\end{proof}

\begin{proof}{\bf of Theorem \ref{thm:MTL}} Let $h^{\ast }$ and $f_{1}^{\ast
},...,f_{T}^{\ast }$ be the minimizers in the definition of $\mathcal{E}_{%
\mathrm{avg}}^{\ast }$. Then 
\begin{eqnarray*}
&&\mathcal{E}_{\mathrm{avg}}(\hat{h},\hat{f}_{1},...,\hat{f}_{T})-\mathcal{E}%
_{\mathrm{avg}}^{\ast } \\
&=&\left( \mathcal{E}_{\mathrm{avg}}(\hat{h},\hat{f}_{1},...,\hat{f}_{T})-%
\frac{1}{nT}\sum_{ti}\ell (\hat{f}_{t}(\hat{h}(X_{ti})),Y_{ti})\right) \\
&+&\left( \frac{1}{nT}\sum_{ti}\ell (\hat{f}_{t}(\hat{h}(X_{ti})),Y_{ti})-%
\frac{1}{nT}\sum_{ti}\ell \left( f_{t}^{\ast }\left( h^{\ast }\left(
X_{ti}\right) \right) ,Y_{ti}\right) \right) \\
&+&\left( \frac{1}{nT}\sum_{ti}\ell \left( f_{t}^{\ast }\left( h^{\ast
}\left( X_{ti}\right) \right) ,Y_{ti}\right) -\frac{1}{T}\sum_{t}\mathbb{E}%
_{\left( X,Y\right) \sim \mu _{t}}\ell \left( f_{t}^{\ast }\left( h^{\ast
}\left( X\right) \right) ,Y\right) \right) .
\end{eqnarray*}%
The last term involves only the $nT$ random variables $\ell \left(
f_{t}^{\ast }\left( h^{\ast }\left( X_{ti}\right) \right) ,Y_{ti}\right) $
with values in $\left[ 0,1\right] $. It can be bounded with probability $%
1-\delta /2$ by $\sqrt{\ln \left( 2/\delta \right) /\left( 2Tn\right) }$
using Hoeffding's inequality. The middle term is non-positive by definition
of $\hat{h},\hat{f}_{1},...,\hat{f}_{T}$ being the corresponding minimizers.
There remains the first term which we bound by 
\begin{equation*}
\sup_{h\in \mathcal{H},f_{1},...,f_{T}\in \tciFourier }\mathcal{E}_{\mathrm{%
avg}}\left( h,f_{1},...,f_{T}\right) -\frac{1}{Tn}\sum_{ti}\ell \left(
f_{t}\left( h\left( X_{ti}\right) \right) ,Y_{ti}\right) .
\end{equation*}%
and appeal to Theorem \ref{thm:unifMTL} to bound the supremum. A union bound
then completes the proof. 
\end{proof}

\bigskip

\subsection{Proof of the Excess Risk Bound for Learning-to-learn}

Recall the definition of the algorithm parametrized by $h\in \mathcal{H}$ 
\begin{equation*}
a\left( h\right) _{\mathbf{Z}}=\arg \min_{f\in \tciFourier }\frac{1}{n}%
\sum_{i}\ell \left( f\left( h\left( X_{i}\right) \right) ,Y_{i}\right) \text{
for }\mathbf{Z}\in \mathcal{Z}^{n}
\end{equation*}%
and the associated minimum $m\left( h\right) _{\mathbf{Z}}$. Also recall
that 
\begin{equation*}
\mathcal{E}_{\eta }\left( h\right) =\mathbb{E}_{\mu \sim \eta }\mathbb{E}_{%
\mathbf{Z\sim }\mu ^{n}}\mathbb{E}_{\left( X,Y\right) \sim \mu }\ell \left(
a\left( h\right) _{\mathbf{Z}}\left( X\right) ,Y\right) 
\end{equation*}%
and the two measures $\mu _{\eta }$ and $\rho _{\eta }$ induced by the
environment $\eta $ and defined by%
\begin{equation*}
\mu _{\eta }\left( A\right) =\mathbb{E}_{\mu \sim \eta }\mu \left( A\right) 
\text{ for }A\subseteq \mathcal{Z}\text{ and }\rho _{\eta }\left( A\right) =%
\mathbb{E}_{\mu \sim \eta }\mu ^{n}\left( A\right) \text{ for }A\subseteq 
\mathcal{Z}^{n}.
\end{equation*}%
Also  recall the definition of $Q^{\prime }$ given in Equation \eqref{eq:Qp}. Again we begin with a uniform bound.

\begin{theorem}
Let $\delta \in (0,1).$ (i) With probability at least $1-\delta $ in $%
\mathbf{\bar{Z}}\sim \rho _{\eta }^{T}$ it holds for every $h\in \mathcal{H}$
that 
\begin{eqnarray}
&&\mathcal{E}_{\eta }\left( h\right) -\frac{1}{T}\sum_{t}m\left( h\right) _{%
\mathbf{Z}_{t}}\leq   \notag \\
&&\frac{\sqrt{2\pi }LG\left( \mathcal{H}\left( \mathbf{\bar{x}}\right)
\right) }{T\sqrt{n}}+\sqrt{2\pi }Q^{\prime }\sup_{h\in \mathcal{H}}\sqrt{%
\frac{\mathbb{E}_{\left( X,Y\right) \sim \mu _{\eta }}\left[ \left\Vert
h\left( X\right) \right\Vert ^{2}\right] }{n}}+\sqrt{\frac{9\ln \left(2/\delta \right) }{2T%
}}.  \notag
\end{eqnarray}

(ii) With probability at least $1-\delta $ in $\mathbf{\bar{Z}}\sim \rho
_{\eta }^{T}$ it holds for every $h\in \mathcal{H}$ that 
\begin{eqnarray}
&&\mathcal{E}_{\eta }\left( h\right) -\frac{1}{T}\sum_{t}m\left( h\right) _{%
\mathbf{Z}_{t}}\leq   \notag \\
&&\frac{\sqrt{2\pi }LG\left( \mathcal{H}\left( \mathbf{\bar{x}}\right)
\right) }{T\sqrt{n}}+\frac{\sqrt{2\pi }Q^{\prime }\sum_{t}\sup_{h\in 
\mathcal{H}}\left\Vert h\left( \mathbf{X}_{t}\right) \right\Vert }{nT}+\sqrt{%
\frac{16\ln \left(4/\delta \right)}{T}}.  \notag
\end{eqnarray}%
\label{thm:unifLTL}
\end{theorem}

\begin{proof}
The key to the proof is the decomposition bound%
\begin{eqnarray}
\sup_{h\in \mathcal{H}}\mathcal{E}_{\eta }\left( h\right) -\frac{1}{T}%
\sum_{t}m\left( h\right) _{\mathbf{Z}_{t}} &\leq &\sup_{h\in \mathcal{H}}%
\mathbb{E}_{\mu \sim \eta }\mathbb{E}_{\mathbf{Z\sim }\mu ^{n}}\left[ 
\mathbb{E}_{\left( X,Y\right) \sim \mu }\ell \left( a\left( h\right) _{%
\mathbf{Z}}\left( X\right) ,Y\right) -m\left( h\right) _{\mathbf{Z}}\right] 
\notag \\
&&+\sup_{h\in \mathcal{H}}\left[ \mathbb{E}_{\mathbf{Z\sim }\rho _{\eta }}%
\left[ m(h)_{\mathbf{Z}}\right] -\frac{1}{T}\sum_{t=1}^{T}m(h)_{\mathbf{Z}%
_{t}}\right] .  \label{Decompo5}
\end{eqnarray}%
In turn we will bound both terms on the right hand side above. A bound on
the second term means that we can predict the empirical risk on the data of
a future task uniformly in $h$. A bound on the first term means that we can
predict the true risk from the empirical risk on the future task. 

We first bound the second term in the right hand side of (\ref{Decompo5}), and use
Theorem \ref{Theorem Generalization}-(ii) on the class of functions 
\begin{equation*}
\left\{ \mathbf{z}\in \mathcal{Z}^{n}\mapsto m\left( h\right) _{\mathbf{z}%
}:h\in \mathcal{H}\right\} 
\end{equation*}%
to get with probability at least $1-\delta $ in $\mathbf{\bar{Z}\sim }\rho
_{\eta }^{T}$ that 
\begin{equation*}
\sup_{h\in \mathcal{H}}\left[ \mathbb{E}_{\mathbf{Z\sim }\rho _{\eta }}\left[
m\left( h\right) _{\mathbf{Z}}\right] -\frac{1}{T}\sum_{t=1}^{T}m\left(
h\right) _{\mathbf{Z}_{t}}\right] \leq \frac{\sqrt{2\pi }}{T}G\left(
S\right) +\sqrt{\frac{9\ln \left(2/\delta \right) }{2T}},
\end{equation*}%
where $S$ is the subset of $\mathbb{R}^{T}$ defined by 
\begin{equation*}
S=\left\{ \left( m\left( h\right) _{\mathbf{Z}_{1}},...,m\left( h\right) _{%
\mathbf{Z}_{T}}\right) :h\in \mathcal{H}\right\} .
\end{equation*}%
We will bound the Gaussian average of $S$ using Slepian's inequality
(Theorem \ref{Slepian Lemma}). Define two Gaussian processes indexed by $%
\mathcal{H}$ as 
\begin{equation*}
\Omega _{h}=\sum_{t}\gamma _{t}m\left( h\right) _{\mathbf{z}_{t}}\text{ and ~%
}\Xi _{h}=\frac{L}{\sqrt{n}}\sum_{kti}\gamma _{kti}h_{k}\left( x_{ti}\right)
.
\end{equation*}%
Now for any $\mathbf{z}\in \mathcal{Z}^{n}$ and representations $h,h^{\prime
}\in \mathcal{H}$ 
\begin{eqnarray*}
\left( m\left( h\right) _{\mathbf{z}}-m\left( h^{\prime }\right) _{\mathbf{z}%
}\right) ^{2} &=&\left( \min_{f\in \tciFourier }\frac{1}{n}\sum_{i}\ell
\left( f\left( h\left( x_{i}\right) \right) ,y_{i}\right) -\min_{f\in
\tciFourier }\frac{1}{n}\sum_{i}\ell \left( f\left( h^{\prime }\left(
x_{i}\right) \right) ,y_{i}\right) \right) ^{2} \\
&\leq &\left( \sup_{f\in \tciFourier }\frac{1}{n}\sum_{i}\ell \left( f\left(
h\left( x_{i}\right) \right) ,y_{i}\right) -\ell \left( f\left( h^{\prime
}\left( x_{i}\right) \right) ,y_{i}\right) \right) ^{2} \\
&\leq &\frac{1}{n}\sup_{f\in \tciFourier }\sum_{i}\left( \ell \left( f\left(
h\left( x_{i}\right) \right) ,y_{i}\right) -\ell \left( f\left( h^{\prime
}\left( x_{i}\right) \right) ,y_{i}\right) \right) ^{2} \\
&\leq &\frac{L^{2}}{n}\sum_{ki}\left( h_{k}\left( x_{i}\right)
-h_{k}^{\prime }\left( x_{i}\right) \right) ^{2},
\end{eqnarray*}%
where in the last step we used the Lipschitz properties of the loss function 
$\ell $ and of the members in the class $\tciFourier $. It follows that 
\begin{eqnarray*}
\mathbb{E}\left( \Omega _{h}-\Omega _{h^{\prime }}\right) ^{2}
&=&\sum_{t}\left( m\left( h\right) _{\mathbf{z}_{t}}-m\left( h^{\prime
}\right) _{\mathbf{z}_{t}}\right) ^{2} \\
&\leq &\frac{L\left( \tciFourier \right) ^{2}}{n}\sum_{kti}\left(
h_{k}\left( x_{ti}\right) -h_{k}^{\prime }\left( x_{ti}\right) \right) ^{2}=%
\mathbb{E}\left( \Xi _{h}-\Xi _{h^{\prime }}\right) ^{2},
\end{eqnarray*}%
so by Theorem \ref{Slepian Lemma} 
\begin{equation*}
G\left( S\right) =\mathbb{E}\sup_{k}\Omega _{h}\leq \mathbb{E}\sup_{k}\Xi
_{h}=\frac{L}{\sqrt{n}}G\left( \mathcal{H}\left( \mathbf{\bar{x}}\right)
\right) .
\end{equation*}%
The second term in the right hand side of (\ref{Decompo5}) is thus bounded with
probability $1-\delta $ by 
\begin{equation}
\frac{\sqrt{2\pi }LG\left( \mathcal{H}\left( \mathbf{\bar{x}}\right) \right) 
}{T\sqrt{n}}+\sqrt{\frac{9\ln \left( 2/\delta \right) }{2T}}.  \label{eq:gino2}
\end{equation}

We now bound the first term on the right hand side of (\ref{Decompo5}) by 
\begin{eqnarray*}
&&\sup_{h\in \mathcal{H}}\mathbb{E}_{\mu \sim \eta }\mathbb{E}_{\mathbf{%
Z\sim }\mu ^{n}}\left[ \mathbb{E}_{\left( X,Y\right) \sim \mu }\ell \left(
a\left( h\right) _{\mathbf{X}}\left( X\right) ,Y\right) -m\left( h\right) _{%
\mathbf{Z}}\right]  \\
&\leq &\sup_{h\in \mathcal{H}}\mathbb{E}_{\mu \sim \eta }\mathbb{E}_{\mathbf{%
Z\sim }\mu ^{n}}\sup_{f\in \tciFourier }\left[ \mathbb{E}_{\left( X,Y\right)
\sim \mu }\ell \left( f\left( h\left( X\right) \right) ,Y\right) -\frac{1}{n}%
\sum_{i}\ell \left( f\left( h\left( X_{i}\right) \right) ,Y_{i}\right) %
\right] .
\end{eqnarray*}%
For $\mathbf{Z}=\left( \mathbf{X},\mathbf{Y}\right) \in \mathcal{Z}^{n}$ and 
$h\in \mathcal{H}$ denote with $\ell \left( \tciFourier \circ h\left( 
\mathbf{X}\right) ,\mathbf{Y}\right) $ the subset of $%
\mathbb{R}
^{n}$ defined by%
\begin{equation*}
\ell \left( \tciFourier \left( h\left( \mathbf{X}\right) \right) ,\mathbf{Y}%
\right) =\left\{ \left( \ell \left( f\left( h\left( X_{i}\right) \right)
,Y_{i}\right) \right) :f\in \tciFourier \right\} .
\end{equation*}%
Using Theorem \ref{Theorem Generalization}-(i) and the contraction lemma,
Corollary \ref{Contraction Lemma}, we can bound the last expression above by%
\begin{eqnarray*}
&&\sup_{h\in \mathcal{H}}\mathbb{E}_{\mu \sim \eta }\mathbb{E}_{\mathbf{%
Z\sim }\mu ^{n}}\sup_{f\in \tciFourier }\left[ \mathbb{E}_{\left( X,Y\right)
\sim \mu }\ell \left( f\left( h\left( X\right) \right) ,Y\right) -\frac{1}{n}%
\sum_{i}\ell \left( f\left( h\left( X_{i}\right) \right) ,Y_{i}\right) %
\right]  \\
&\leq &\sqrt{2\pi }\sup_{h\in \mathcal{H}}\mathbb{E}_{\mathbf{Z\sim \rho }%
_{\eta }}\frac{G\left( \ell \left( \tciFourier \left( h\left( \mathbf{X}%
\right) \right) ,\mathbf{Y}\right) \right) }{n} \\
&\leq &\sqrt{2\pi }\sup_{h\in \mathcal{H}}\mathbb{E}_{\mathbf{Z\sim \rho }%
_{\eta }}\frac{G\left( \tciFourier \left( h\left( \mathbf{X}\right) \right)
\right) }{n} \\
&= &\frac{\sqrt{2\pi }}{n}\sup_{h\in \mathcal{H}}\mathbb{E}_{\mathbf{%
Z\sim \rho }_{\eta }}\frac{G\left( \tciFourier \left( h\left( \mathbf{X}%
\right) \right) \right) }{\left\Vert h\left( \mathbf{X}\right) \right\Vert }%
\left\Vert h\left( \mathbf{X}\right) \right\Vert  \\
&\leq &\frac{\sqrt{2\pi }}{n}Q^{\prime }\sup_{h\in \mathcal{H}}\mathbb{E}_{%
\mathbf{Z\sim \rho }_{\eta }}\left\Vert h\left( \mathbf{X}\right)
\right\Vert ,
\end{eqnarray*}%
using Hoelder's inequality and the definition of $Q^{\prime }$ in the last
step. But, using Jensen's inequality,%
\begin{equation*}
\mathbb{E}_{\mathbf{Z\sim \rho }_{\eta }}\left\Vert h\left( \mathbf{X}%
\right) \right\Vert \leq \sqrt{\mathbb{E}_{\mathbf{Z}\sim \rho _{\eta
}}\sum_{i}\left\Vert h\left( X_{i}\right) \right\Vert ^{2}}=\sqrt{n~\mathbb{E%
}_{\left( X,Y\right) \sim \mu _{\eta }}\left\Vert h\left( X\right)
\right\Vert ^{2}},
\end{equation*}%
since $\mathbf{Z}\sim \rho _{\eta }$ is iid. Inserting this in the previous
chain of inequalities and combining with (\ref{eq:gino2}) gives the first
part of the theorem.

To obtain the data dependent bound we use the fact that, with probability at
least $1-\delta /4$, 
\begin{eqnarray}
\sup_{h\in \mathcal{H}}\mathbb{E}_{\mathbf{Z\sim \rho }_{\eta }}\frac{%
G\left( \ell \left( \tciFourier \left( h\left( \mathbf{X}\right) \right) ,%
\mathbf{Y}\right) \right) }{n} &\leq &\mathbb{E}_{\mathbf{X\sim }\rho _{\eta
}}\sup_{h\in \mathcal{H}}\frac{G\left( \ell \left( \tciFourier \left(
h\left( \mathbf{X}\right) \right) ,\mathbf{Y}\right) \right) }{n} \\
&\leq &\frac{1}{T}\sum_{t}\sup_{h\in \mathcal{H}}\frac{G\left( \ell \left(
\tciFourier \left( h\left( \mathbf{X}_{t}\right) \right) ,\mathbf{Y}%
_{t}\right) \right) }{n}+\sqrt{\frac{\ln \left(4/\delta \right)}{2T}}  \label{eq:gino1}
\end{eqnarray}%
The last inequality follows from Hoeffding's inequality since for any $h\in 
\mathcal{H}$ and any sample $\mathbf{Z}\in \mathcal{Z}^{n}$ 
\begin{eqnarray*}
0 &\leq &\frac{G\left( \ell \left( \tciFourier \left( h\left( \mathbf{X}%
\right) \right) ,\mathbf{Y}\right) \right) }{n}=\frac{1}{n}\mathbb{E}%
_{\gamma }\sup_{f}\sum_{i}\gamma _{i}\ell \left( f\left( h\left(
X_{i}\right) \right) ,Y_{i}\right)  \\
&\leq &\frac{1}{n}\mathbb{E}_{\gamma }\left( \sum_{i}\gamma _{i}^{2}\right)
^{1/2}\sup_{f}\left( \sum_{i}\ell \left( f\left( h\left( X_{i}\right)
\right) ,Y_{i}\right) ^{2}\right) ^{1/2}\leq 1,
\end{eqnarray*}%
where we have also used the fact that the loss function $\ell $ has range in 
$[0,1]$. Bounding 
\begin{equation*}
G\left( \ell \left( \tciFourier \circ h\left( \mathbf{X}_{t}\right) ,\mathbf{%
Y}_{t}\right) \right) \leq Q^{\prime }\left\Vert h\left( \mathbf{X}%
_{t}\right) \right\Vert 
\end{equation*}%
as above and combining (\ref{eq:gino1}) and (\ref{eq:gino2}) in (\ref%
{Decompo5}) with a union bound gives the second inequality of the theorem.

\end{proof}

Remark: In the proof of the fully data-dependent part above the bound on 
\begin{equation*}
\sup_{h\in \mathcal{H}}\mathbb{E}_{\mathbf{Z\sim \rho }_{\eta }}\frac{%
G\left( \ell \left( \tciFourier \left( h\left( \mathbf{X}\right) \right) ,%
\mathbf{Y}\right) \right) }{n}
\end{equation*}%
is very crude. Instead we could have again invoked Theorem \ref{Theorem
Generalization} to get a better bound with a more complicated expression
involving nested Gaussian averages. We have chosen the simpler path for
greater clarity.\bigskip

\begin{proof}{\bf of Theorem \ref{thm:LTL}} Recall that 
\begin{equation*}
\mathcal{E}_{\eta }^{\ast }=\min_{h\in \mathcal{H}}\mathbb{E}_{\mu \sim \eta
}\left[ \min_{f\in \tciFourier }\mathbb{E}_{\left( X,Y\right) \sim \mu }\ell
\left( f\left( h\left( X\right) \right) ,Y\right) \right] .
\end{equation*}%
We denote with $h^{\ast }$ the minimizer in $\mathcal{H}$ occurring in the
definition of $\mathcal{E}_{\eta }^{\ast }$. We have the following
decomposition 
\begin{eqnarray}
\mathcal{E}_{\eta }(\hat{h})-\mathcal{E}_{\eta }^{\ast } &=&\left( \mathcal{E%
}_{\eta }(\hat{h})-\frac{1}{T}\sum_{t}m(\hat{h})_{\mathbf{Z}_{t}}\right)
\label{decompo 1} \\
&&+\left( \frac{1}{T}\sum_{t}m(\hat{h})_{\mathbf{Z}_{t}}-\frac{1}{T}%
\sum_{t}m(h^{\ast })_{\mathbf{Z}_{t}}\right)  \label{decompo 2} \\
&&+\left( \frac{1}{T}\sum_{t}m(h^{\ast })_{\mathbf{Z}_{t}}-\mathbb{E}_{%
\mathbf{Z}\sim \rho _{\eta }}\left[ m\left( h^{\ast }\right) _{\mathbf{Z}}%
\right] \right)  \label{decompo 3} \\
&&+\mathbb{E}_{\mu \sim \eta }\left[ \mathbb{E}_{\mathbf{Z}\sim \mu ^{n}}%
\left[ m\left( h^{\ast }\right) _{\mathbf{Z}}\right] -\min_{f\in \tciFourier
}\mathbb{E}_{\left( X,Y\right) \sim \mu }\ell \left( f\left( h^{\ast }\left(
X\right) \right) ,Y\right) \right] .  \label{decompo 4}
\end{eqnarray}%
For a fixed distribution $\mu $ let $f_{\mu }^{\ast }$ be the minimizer in $%
\min_{f\in \tciFourier }\mathbb{E}_{\left( X,Y\right) \sim \mu }\ell \left(
f\left( h^{\ast }\left( X\right) \right) ,Y\right) $. By definition of $%
m\left( h^{\ast }\right) _{\mathbf{Z}}$ we have for every $\mu \sim \eta $
that 
\begin{eqnarray*}
\mathbb{E}_{\mathbf{Z}\sim \mu ^{n}}\left[ m\left( h^{\ast }\right) _{%
\mathbf{Z}}\right] &=&\mathbb{E}_{\mathbf{Z}\sim \mu ^{n}}\min_{f\in
\tciFourier }\frac{1}{m}\sum_{i=1}^{m}\ell \left( f\left( h^{\ast }\left(
X_{i}\right) \right) ,Y_{i}\right) \\
&\leq &\mathbb{E}_{\mathbf{Z}\sim \mu ^{n}}\frac{1}{m}\sum_{i=1}^{m}\ell
\left( f_{\mu }^{\ast }\left( h^{\ast }\left( X_{i}\right) \right)
,Y_{i}\right) \\
&=&\mathbb{E}_{\left( X,Y\right) \sim \mu }\ell \left( f_{\mu }^{\ast
}\left( h^{\ast }\left( X\right) \right) ,Y\right) ,
\end{eqnarray*}%
since $\mathbf{Z}$ is iid. The term in (\ref{decompo 4}) is therefore
non-positive.

The term in (\ref{decompo 3}) involves the deviation of the empirical and
true averages of the $T$ iid $\left[ 0,1\right] $-valued random variables $%
m\left( h^{\ast }\right) _{\mathbf{Z}_{t}}$. With Hoeffding's inequality
this can be bounded with probability at least $1-\delta /8$ by $\sqrt{\ln
\left( 8/\delta \right) /\left( 2T\right) }$. The term \eqref{decompo 2} is
non-positive by the definition of $\hat{h}$.

There remains the term (\ref{decompo 1}), which we bound by Theorem \ref%
{thm:unifLTL}. The result now follows by combining this bound with the bound
on (\ref{decompo 3}) in a union bound and some numerical simplifications. %
\end{proof}

\section{Conclusion}
\label{sec:discussion}
Several works have advocated that sharing features among tasks as a means to learning representations which capture invariant properties to tasks can be highly beneficial. In this paper, we studied the statistical properties of a general MTRL method, presenting bounds on its learning performance in both settings of MTL and LTL. Our work provides a rigorous justification of the benefit offered by MTRL over learning the tasks independently. To give the paper a clear focus we have illustrated this advantage in the case of linear feature learning. Our results however apply to fairly general classes of representations ${\cal H}$ and specifications ${\cal F}$, and 
similar conclusions may be derived for other nonlinear MTRL methods. We conclude by sketching specific cases which deserve a separated study:
\begin{itemize}
\item {\em Deep networks.} As we noted our bounds directly apply to multilayer, deep architectures obtained by iteratively composing linear transformations with nonlinear activation functions, such as the rectifier linear unit or the sigmoid functions. The representations learned by such methods tend to be specific in that only a subset of components are ``active'' on each given input, which makes our bounds particularly attractive for further analysis.

\item {\em Sparse coding.} Another interesting case of our framework is obtained when the specialized class ${\cal F}$ consists of sparse linear predictors. This case has been considered in \cite{MPR,RE14} when the representation class consists of linear functions. Different choices of sparse classes ${\cal F}$ could lead to interesting learning methods.

\item {\em Representations in RKHS}. As we already noted the feature maps forming the class ${\cal H}$ could be vector-valued functions in a reproducing kernel Hilbert space. Although kernel methods are more difficult to apply to large datasets required for MTRL and need additional approximation steps, the representations learned using for example Gaussian kernels would be very specific and suitable for our bounds.
\end{itemize}

\acks{We thank the reviewers for their helpful comments.}

\end{document}